\documentclass{IEEEtran}
\usepackage{graphicx}
\usepackage{subcaption}
\usepackage[]{threeparttable}
\usepackage{tikz}
\usepackage{multirow} 
\usepackage{booktabs}
\usepackage{amssymb}
\usepackage{amsthm}
\usepackage{bm}
\usepackage{amsmath}
\usepackage{mathrsfs}
\newtheorem{assumption}{Assumption}
\newtheorem{theorem}{Theorem}[section]
\newtheorem{lemma}[theorem]{Lemma}
\newtheorem{proposition}[theorem]{Proposition}

\theoremstyle{definition} 

\theoremstyle{remark} 
\newtheorem{remark}[theorem]{Remark}
\newcommand{\Exp}{\mathbb{E}}
\DeclareMathOperator*{\argmin}{arg\,min} 
\DeclareMathOperator*{\argmax}{arg\,max} %

\usepackage{algorithm}
\usepackage{algpseudocode}

\begin{document}

\title{Natural Gradient Bayesian Filtering:
\\Geometry-Aware Filter for Dynamical Systems}

\author{Chang Liu$^{2}$, Wenhan Cao$^{1}$, Zeju Sun$^{3}$, Tianyi Zhang$^{1}$, Jiayu Yuan$^{2}$, Yi Zeng$^{2}$, Ting Yuan$^{4}$, Yao Lyu$^{1}$, Wei Wu$^{5}$, Stephen Shing-Toung Yau$^{6}$, Shengbo Eben Li$^{1}$\thanks{Corresponding author: Shengbo Eben Li.\\
$^{1}$School of Vehicle and Mobility, Tsinghua University, China.\\
$^{2}$School of Advanced Manufacturing and Robotics, Peking University, China.\\
$^{3}$Beijing Institute of Mathematical Sciences and Applications (BIMSA), China.\\
$^{4}$School of Automation and Intelligent Sensing, Shanghai Jiao Tong University, China.\\
$^{5}$State Key Laboratory of Intelligent Green Vehicle and Mobility, Tsinghua University.\\
$^{6}$Department of Mathematical Sciences, Tsinghua University, China.\\
This study is supported by National Science and Technology Major Project (No 2025ZD1606200), Beijing Natural Science Foundation with L257002, and NSF China with 92582205, 92567301.}}

\maketitle



\begin{abstract}

Bayesian filtering is a cornerstone of state estimation in complex systems such as aerospace systems, yet exact solutions are available only for linear Gaussian models. In practice, nonlinear systems are handled through tractable approximations, with Gaussian filters such as the extended and unscented Kalman filters being among the most widely used methods. This tutorial revisits Gaussian filtering from an information-geometric perspective, viewing the prediction and measurement update steps as inference procedures over state distributions. Within this framework, we introduce a geometry-aware Gaussian filtering approach that leverages natural gradient descent on the statistical manifold of Gaussian distributions. The resulting Natural Gradient Gaussian Approximation (NANO) filter iteratively refines the posterior mean and covariance while respecting the intrinsic geometry of the Gaussian family and preserving the positive definiteness of the covariance matrix. We further highlight fundamental connections to the classical Kalman filtering, showing that a single natural-gradient step exactly recovers the Kalman measurement update in the linear-Gaussian case. The practical implications of the proposed framework are illustrated through case studies in representative nonlinear estimation problems, including satellite attitude estimation, simultaneous localization and mapping, and state estimation for robotic systems including quadruped and humanoid robots.

\end{abstract}

\section{Introduction}

\IEEEPARstart{S}{tate} estimation concerns inferring the latent state of a dynamical system from noisy and partial observations, and it underpins a broad spectrum of modern applications ranging from astrophysics and robotics to power grids, manufacturing, and transportation. A principled way to formalize this inference task is sequential Bayesian estimation: at each time instant, one aims to characterize the conditional distribution of the current state given the entire measurement history, i.e., the filtering posterior. Operationally, this posterior is propagated in time through a recursion that alternates between (a) propagating uncertainty through the state dynamics and (b) incorporating new information carried by measurements \cite{chen2003bayesian}. The first part is governed by the Chapman--Kolmogorov integral, which transports a distribution forward via the transition kernel describing  state evolution; the second part applies Bayes' rule to combine the predicted distribution with the measurement likelihood that encodes the observation mechanism \cite{liu2017measurement}.

This recursion admits a closed-form solution only in special cases. In particular, when both the dynamics and the measurement model are linear and all uncertainties are Gaussian, the filtering posterior remains Gaussian over time and is completely characterized by its first two moments. The resulting recursion is the Kalman filter (KF) \cite{kalman1960new,liu2017measurement}, whose tractability relies on Gaussian closure under linear mappings and conjugacy under conditioning. Once either nonlinearity or non-Gaussianity is present, these properties generally fail, and the exact Bayes filter becomes analytically inaccessible \cite{sarkka2023bayesian}.

The practical implication is that nonlinear/non-Gaussian filtering is, in most settings, an approximation problem: one must decide how to represent and update uncertainty. Particle filtering offers a conceptually general solution by representing the posterior with a weighted ensemble of samples, yielding a discrete empirical approximation to the continuous density \cite{elfring2021particle}. While such methods can be asymptotically consistent, accurate performance in challenging regimes often demands many particles, which in turn induces a computational burden that may be incompatible with real-time constraints \cite{elfring2021particle}. A different and widely used compromise is to restrict the posterior family to Gaussians at each time step, giving rise to the class of Gaussian filters, which are often preferred in practice due to their favorable computational profile \cite{thrun2002probabilistic}.

Many Gaussian filters can be interpreted through a common construction: rather than tackling the nonlinear model directly, one first constructs a linear--Gaussian surrogate of the state-space model, which is typically an affine approximation with additive Gaussian noise, and then executes KF-style moment updates on that surrogate \cite{sarkka2023bayesian}. The essential distinction among popular Gaussian filters is therefore the linearization mechanism used to build the surrogate. The classical route is Jacobian-based local linearization via a first-order Taylor expansion, which leads to the extended Kalman filter (EKF) \cite{smith1962application,mcelhoe1966assessment}. 
The iterated EKF (IEKF) \cite{gelb1974applied} modifies this by repeatedly re-linearizing during the measurement update, using successive posterior-mean estimates; the resulting procedure is closely related to Gauss--Newton iterations and can be viewed as seeking a local maximum-a-posteriori solution \cite{bell1993iterated}.

However, point-based Taylor linearization has a structural weakness: it is driven primarily by a local mean expansion and does not preserve how covariance is transformed by nonlinear mappings. This motivates moment-based linearization, where the affine surrogate is chosen to reproduce, as closely as possible, the mean and covariance of the nonlinear transformation under a Gaussian input. Such constructions can be cast as statistical linear regression (SLR), in which the affine parameters minimize an expected regression error \cite{sarkka2023bayesian}. Implementing SLR entails evaluating Gaussian-weighted integrals, and different numerical integration rules yield different Gaussian filters: the unscented transform underlies the UKF \cite{julier1995new}, Gauss--Hermite quadrature yields the Gauss--Hermite KF \cite{arasaratnam2007discrete}, and spherical cubature rules lead to the cubature KF \cite{arasaratnam2009cubature}. More recently, the posterior linearization filter (PLF) applies SLR with respect to the posterior distribution during the update rather than with respect to the prior, with the goal of improving update accuracy \cite{sarkka2023bayesian}.

This ``linearize-then-KF'' paradigm---namely, building a linear--Gaussian surrogate model so that a Kalman-style update becomes applicable---is commonly termed the \textit{enabling approximation} framework \cite{sarkka2023bayesian}. While enabling approximation is ubiquitous in practical Gaussian filtering, its optimality in the measurement update is questionable. In particular, a key question is whether the update obtained by enabling approximation coincides with optimal Gaussian approximation to the true Bayesian posterior. The recent result in \cite{cao2024nonlinear,Cao2025AlgorithmDA} clarifies that, in contrast to the prediction stage where moment-matching rules employed by UKF/Gauss--Hermite KF/cubature KF admit an optimization-consistent interpretation \cite{julier1995new,arasaratnam2007discrete,arasaratnam2009cubature}, the update update in those methods does not generally recover the exact solution. Their argument relies on an optimization viewpoint, in which the prediction and update recursions are characterized as solutions to two different optimization problems, thereby exposing a structural mismatch introduced by the surrogate-based update.

Motivated by the optimization viewpoint of Bayesian filtering, a natural way to construct Gaussian filters is to restrict the filtering posterior to the Gaussian family and then compute the best member of this family by solving an explicit optimization problem. Under this perspective, the measurement update is no longer implemented by surrogate linearization and a Kalman-style recursion, but instead by directly optimizing the Gaussian parameters so that the resulting approximation matches the true Bayesian posterior as closely as possible according to a chosen criterion. In fact, this idea has been implemented in a class of \emph{gradient-based Gaussian filters} that use gradient descent and its variants to iteratively refine the Gaussian approximation; for example, Gultekin and Paisley \cite{gultekin2017nonlinear} propose stochastic gradient schemes in which Monte Carlo sampling is used to approximate gradients that are otherwise difficult to evaluate exactly.

A key challenge in this line of work is that straightforward gradient descent on Gaussian parameters is not geometry-aware. In particular,  Gaussian distributions do not form a flat Euclidean parameter space; rather, they possess a natural information-geometric structure, under which the meaning of the steepest descent direction depends on the intrinsic metric of the distribution family. As a consequence, Euclidean-gradient updates can be sensitive to the chosen parameterization, may suffer from poor scaling, and often demand careful step-size tuning to achieve stable progress. These observations motivate the use of natural-gradient methods, which incorporate the geometry of the statistical manifold and produce descent directions that are invariant to reparameterization. In practice, geometry-aware updates frequently improve numerical robustness and convergence behavior, while offering a principled way to respect the structural constraints inherent in Gaussian approximations.

Motivated by these considerations, this tutorial aims to place Gaussian filtering within a unified optimization and information-geometric framework. Its main goal is to provide an intuitive perspective on how geometry-aware updates arise naturally in approximate Bayesian filtering and how they can be translated into practical algorithms. The main contributions of this article are summarized as follows:
\begin{itemize}
    \item We present an information-geometric interpretation of Gaussian filtering, viewing prediction and measurement update as projection problems on the statistical manifold of Gaussian distributions. This perspective clarifies the role of approximation in nonlinear Bayesian filtering and highlights the geometric structure underlying commonly used Gaussian filters.
    \item Within the proposed framework, the measurement update is interpreted as a geometry-aware optimization problem. This leads to a natural gradient descent update that respects the intrinsic structure of Gaussian distributions, maintains positive-definite covariance representations, and remains invariant to parameterization.
    \item To support practical understanding and adoption, the article presents representative case studies, including satellite attitude estimation, simultaneous localization and mapping (SLAM), and state estimation of quadruped and humanoid robot .
\end{itemize}

The structure of this article is as follows. 
Section \ref{sec.background} presents the optimization perspective on Bayesian filtering and primer on information geometry. 
Section \ref{sec.ngf} reviews the development of NANO-filter utilizing natural gradient descent. 
Representative case studies will be introduced in Section \ref{sec.example}, and we finally draw conclusions in Section \ref{sec.conclusion}.





\section{Background}\label{sec.background}

\subsection{Bayesian Filtering Recap}
Consider the following nonlinear discrete-time stochastic system:
\begin{equation}\label{eq.SSM}
\begin{aligned}
x_{t+1} &= f(x_t) + \xi_t,\\
y_t &= g(x_t) + \zeta_t,
\end{aligned}
\end{equation}
where $x_t \in \mathbb{R}^n$ denotes the latent state and $y_t \in \mathbb{R}^m$ is the corresponding noisy measurement. The mapping $f:\mathbb{R}^n \to \mathbb{R}^n$ is the transition function that governs the state evolution, while $g:\mathbb{R}^n \to \mathbb{R}^m$ is the measurement function that specifies the measurement mechanism. The random terms $\xi_t$ and $\zeta_t$ represent process noise and measurement noise, respectively. A standard assumption is that the initial state $x_0$, the process noise sequence $\{\xi_t\}$, and the measurement noise sequence $\{\zeta_t\}$ are mutually independent, and that both noise sequences are independent over time.

An equivalent probabilistic description of 
\eqref{eq.SSM} is given by the hidden Markov model (HMM)
\begin{equation}\label{eq.HMM}
\begin{aligned}
x_0 &\sim p(x_0),\\
x_t &\sim p(x_t | x_{t-1}),\\
y_t &\sim p(y_t | x_t),
\end{aligned}
\end{equation}
where $p(x_t | x_{t-1})$ and $p(y_t | x_t)$ are the transition and emission densities, respectively, and $p(x_0)$ is the distribution of the initial state. In fact, \eqref{eq.SSM} and \eqref{eq.HMM} are different formulations that describe the same model. For example, if the transition model is
\[
x_t = A x_{t-1} + \xi_{t-1}, \quad \xi_{t-1} \sim \mathcal{N}(0,Q),
\]
with $Q$ denoting the process-noise covariance, then the corresponding transition density is
\[
p(x_t | x_{t-1}) = \mathcal{N}(x_t; A x_{t-1}, Q).
\]

The objective of state estimation is to infer $x_t$ from the noisy observations $y_{1:t}$. In general, an optimal estimator is obtained in two steps: (i) compute the posterior distribution $p(x_t | y_{1:t})$, and (ii) extract a point estimate $\hat{x}_{t | t}$ from this posterior according to a chosen criterion, such as minimum mean-square error (MMSE) or maximum a posteriori (MAP) estimation.

A principled recursion for computing $p(x_t | y_{1:t})$ is provided by Bayesian filtering, which alternates between prediction and update:
\begin{subequations}\label{eq.BF}
\begin{align}
p(x_t | y_{1:t-1})
&= \int p(x_t | x_{t-1})\, p(x_{t-1} | y_{1:t-1}) \, \mathrm{d}x_{t-1},
\label{eq.prediction}\\
p(x_t | y_{1:t})
&= \frac{p(y_t | x_t)\, p(x_t | y_{1:t-1})}
{\int p(y_t | x_t)\, p(x_t | y_{1:t-1}) \, \mathrm{d}x_t }.
\label{eq.update}
\end{align}
\end{subequations}
The prediction step \eqref{eq.prediction} uses the Chapman--Kolmogorov equation to propagate the previous posterior through the transition model to form the prior $p(x_t | y_{1:t-1})$. The update step \eqref{eq.update} then incorporates the new measurement via Bayes' rule, where $p(y_t | x_t)$ acts as the likelihood. Once the filtering posterior is available, standard decision rules can be applied to obtain the state estimate.

\subsection{An Optimization Perspective on Bayesian Filtering}

Motivated by the optimization perspective of Bayesian inference in \cite{knoblauch2022optimization}, the prediction and update steps of Bayesian filtering also admit variational characterizations. Throughout this section, $q:\mathbb{R}^n\to\mathbb{R}_+$ denotes a candidate probability density.

Let us first introduce the optimization perspective of the prediction step of Bayesian filtering. Recall that the prior density is
\[
p(x_t| y_{1:t-1})
=
\int p(x_t| x_{t-1})\,p(x_{t-1}| y_{1:t-1})\,\mathrm{d}x_{t-1}.
\]
In fact, $p(x_t|y_{1:t-1})$ can be recovered from the variational problem
\begin{equation}\label{eq.BF_prediction_optimization_rewrite}
p(x_t| y_{1:t-1})
=
\argmax_{q(x_t)}
\Exp_{p(x_{t-1}| y_{1:t-1})}\Exp_{p(x_t| x_{t-1})}
\Bigl[\log q(x_t)\Bigr].
\end{equation}
To see why, note that adding a $q$-independent constant does not change the maximizer. Hence, for any constant $Z>0$, we have
\[
\begin{aligned}
&\argmax_{q(x_t)}
\Exp_{p(x_{t-1}| y_{1:t-1})}\Exp_{p(x_t| x_{t-1})}
\bigl[\log q(x_t)\bigr]
\\
=&
\argmin_{q(x_t)}
\Exp_{p(x_{t-1}| y_{1:t-1})}\Exp_{p(x_t| x_{t-1})}
\Bigl[\log \frac{Z}{q(x_t)}\Bigr].
\end{aligned}
\]
Choosing
\[
Z \;=\; \int p(x_t| x_{t-1})\,p(x_{t-1}| y_{1:t-1})\,\mathrm{d}x_{t-1},
\]
the objective becomes
\[
\Exp_{p(x_t| y_{1:t-1})}\Bigl[\log \frac{p(x_t| y_{1:t-1})}{q(x_t)}\Bigr]
=
D_{\mathrm{KL}}\!\bigl(p(x_t| y_{1:t-1})\,\|\,q(x_t)\bigr),
\]
which is uniquely minimized when $q(x_t)=p(x_t| y_{1:t-1})$. This yields the characterization in \eqref{eq.BF_prediction_optimization_rewrite}.

We next turn to the measurement update. The posterior distribution is characterized by \cite{cao2023generalized,cao2025robust}
{\small
\begin{equation}\label{eq.BF_update_optimization_rewrite}
\begin{aligned}
&p(x_t| y_{1:t})
\\
=&
\argmin_{q(x_t)}
\Bigl\{
\Exp_{q(x_t)}\bigl[-\log p(y_t| x_t)\bigr]
+
D_{\mathrm{KL}}\!\bigl(q(x_t)\,\|\,p(x_t| y_{1:t-1})\bigr)
\Bigr\}.
\end{aligned}
\end{equation}}
Expanding the KL divergence shows that, up to additive constants independent of $q$, we have
\[
\begin{aligned}
&\Exp_{q(x_t)}\bigl[-\log p(y_t| x_t)\bigr]
+
D_{\mathrm{KL}}\!\bigl(q(x_t)\,\|\,p(x_t| y_{1:t-1})\bigr)
\\
=&
\Exp_{q(x_t)}\Bigl[\log\frac{q(x_t)}{p(y_t| x_t)\,p(x_t| y_{1:t-1})}\Bigr].
\end{aligned}
\]
Then let us introduce the normalization constant
\[
Z \;=\; \int p(y_t| x_t)\,p(x_t| y_{1:t-1})\,\mathrm{d}x_t.
\]
Adding $\log Z$ inside the expectation does not change the minimizer, and the objective becomes equivalent to
\[
\begin{aligned}
&\Exp_{q(x_t)}\Bigl[\log\frac{q(x_t)}{p(y_t| x_t)\,p(x_t| y_{1:t-1})/Z}\Bigr]
\\
=&
D_{\mathrm{KL}}\!\Bigl(
q(x_t)\,\Big\|\,
\frac{p(y_t| x_t)\,p(x_t| y_{1:t-1})}{Z}
\Bigr).
\end{aligned}
\]
Therefore, the minimizer is uniquely attained at
\[
q^\star(x_t)
=
\frac{p(y_t| x_t)\,p(x_t| y_{1:t-1})}
{\int p(y_t| x_t)\,p(x_t| y_{1:t-1})\,\mathrm{d}x_t}
=
p(x_t| y_{1:t}),
\]
which coincides with the Bayesian posterior.


\begin{figure}
    \centering
    \includegraphics[width=0.99\linewidth]{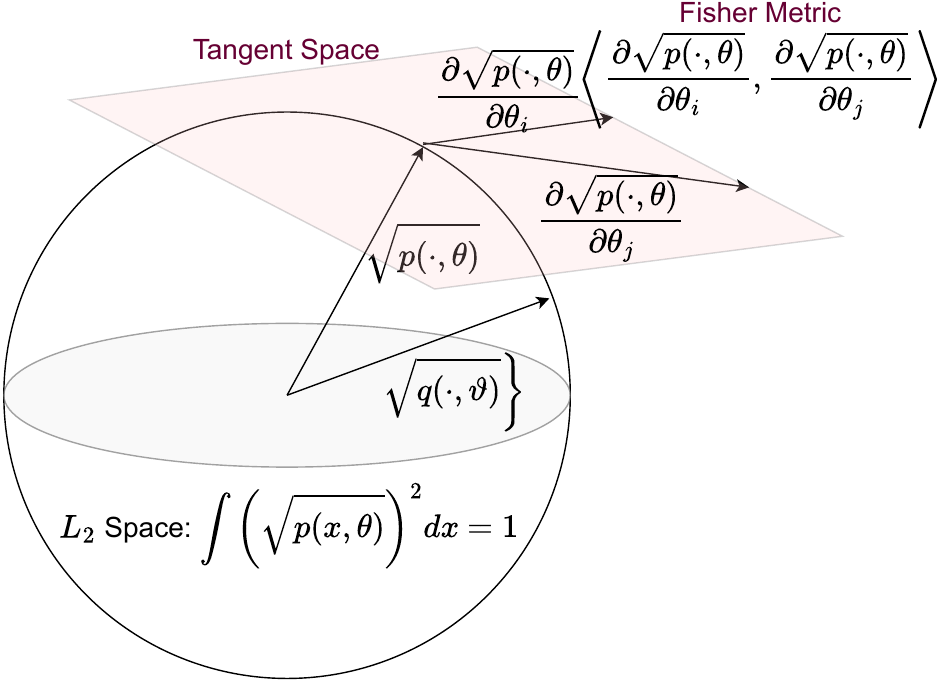}
    \caption{Illustration of Fisher information metric}
    \label{fig.ig}
\end{figure}

\subsection{Information Geometry Primer}
From the dynamical system point of view, the purpose of a filtering algorithm is to find a proper evolution strategy of conditional probability distributions over time, i.e., to determine a computable map:
\begin{equation}
	t \rightarrow \mu_{t} \in \mathcal{P}(\mathbb{R}^{n}),
	\label{eq:5}
\end{equation} 
where $\mathcal{P}(\mathbb{R}^{n})$ is the set of all probability measures on $\mathbb{R}^{n}$ with positive density functions, such that in some sense, the probability measure, $\mu_{t}$, is necessarily close to the solution of Bayesian filter, $p(x_{t}|y_{1:t})$, for each time $t\in \mathbb{N}$.

If we think of the set $\mathcal{P}(\mathbb{R}^{n})$ as an infinite-dimensional manifold, then based on the optimization perspective on Bayesian filtering, the computable map \eqref{eq:5} can be obtained and interpreted in a differential geometric approach. The studies on the geometric interpretation of Bayesian filtering belong to the subject called \textbf{Information Geometry}, in which all the problems involving the probability distributions are treated and solved from a geometric perspective. 

The study of information geometry dates back to the 1940s, when C. R. Rao first formulated the parameter space of probability distributions into a Riemannian manifold, where the Riemannian metric is defined by Fisher information matrix \cite{Rao1945, Rao1948}. The theory of information geometry has been comprehensively developed since 1980s. Readers who are interested in this subject may refer to the important monographs such as \cite{AmariNagaoka2000}.

This geometric treatment of probability distributions inspires more accurate and efficient filter designs, and also make the improvements on the filter performance more theoretically interpretable. Hence, we briefly introduce the basic ideas of information geometry first, before moving on to the design of the novel Gaussian filter.
\subsubsection{Statistical Manifolds}
The manifold structure of the set of probability measures, $\mathcal{P}(\mathbb{R}^{n})$, is rooted in the constraint that the integral of a probability density over the whole space $\mathbb{R}^{n}$ equals one. In order to avoid dealing with infinite-dimensional objects, it is common to consider a finite-dimensional submanifold 
\begin{equation}
	S = \{p(\cdot, \theta): \ \theta\in \Theta\},
\end{equation}
which is parameterized by $\theta\in\Theta$, and the parameter space $\Theta\subset \mathbb{R}^{m}$ is assumed to be an open set in the $m$-dimensional Euclidean space, $\mathbb{R}^{m}$. 

Mathematically, for a given $m$-dimensional differential manifold, it is natural to consider its tangent space (which is an $m$-dimensional Euclidean space at each point of the manifold) and define an inner product on the tangent space. The inner product on the tangent space at each point of the manifold together forms a metric on the manifold, which is called the \textbf{Riemannian metric}, and a manifold with a given Riemannian metric is called a \textbf{Riemannian manifold}.

Specifically, a well-defined metric on the tangent space of manifold of the parameterized probability densities can help us properly measure the distance between two probability distributions and construct an accessible trajectory between them. As a submanifold of the probability measures, it is also demanding that the Riemannian metric will inherit and preserve the geometric properties of $\mathcal{P}(\mathbb{R}^{n})$, such that the manifold will not be `twisted' by the parametrization.

However, it is not easy to determine a `natural' inner product on $\mathcal{P}(\mathbb{R}^{n})$, because we only assume that the probability density functions are integrable, i.e., $\mathcal{P}(\mathbb{R}^{n})\subset L^{1}(\mathbb{R}^{n})$, and $L^{1}(\mathbb{R}^{n})$ is not a Hilbert space. Fortunately, each probability density function in $\mathcal{P}(\mathbb{R}^{n})$ is positive, as previously assumed, and thus, for each density function $p(x)\in \mathcal{P}(\mathbb{R}^{n})$, $\sqrt{p}\in L^{2}(\mathbb{R}^{n})$ is square-integrable, and the space of all square-integrable functions on $\mathbb{R}^{n}$, $L^{2}(\mathbb{R}^{n})$, is a Hilbert space, with inner product
\begin{equation}
	\langle \sqrt{p},\sqrt{q}\rangle = \int \sqrt{p(x)}\sqrt{q(x)}dx,
\end{equation}
for each probability density functions $p, q\in\mathcal{P}(\mathbb{R}^{n})$, see Fig.~\ref{fig.ig}. Moreover, corresponding to this inner product, the so-called \textbf{Hellinger metric} between two probability density functions $p$ and $q$ can be defined as
\begin{equation}
	H(p,q) := \|\sqrt{p} - \sqrt{q}\|_{L^{2}} = \sqrt{\langle \sqrt{p}-\sqrt{q}, \sqrt{p}-\sqrt{q}\rangle }.
\end{equation}

Inspired by the above discussion, let us consider the corresponding set of square roots of densities:
\begin{equation}
	S^{1/2} = \{\sqrt{p(\cdot, \theta)}:\theta\in \Theta\}\subset\ L^{2}(\mathbb{R}^{n}).
\end{equation}
Then, $S^{1/2}$ can be regarded as a finite-dimensional submanifold of the Hilbert space $L^{2}(\mathbb{R}^{n})$, with tangent space at each $\sqrt{p(\cdot, \theta)}$ defined by
\begin{equation}
	L_{\sqrt{p(\cdot, \theta)}}S^{1/2} :=	\left\{\frac{\partial\sqrt{p(\cdot, \theta)}}{\partial \theta_{1}},\cdots,\frac{\partial\sqrt{p(\cdot, \theta)}}{\partial \theta_{m}}\right\},
\end{equation}
and the Riemannian metric is defined as
\begin{equation}
	\begin{aligned} 
	I_{ij}(\theta)& = 4 \left\langle \frac{\partial\sqrt{p(\cdot, \theta)}}{\partial \theta_{i}}, \frac{\partial\sqrt{p(\cdot, \theta)}}{\partial \theta_{j}}\right\rangle \\
	&= \int \frac{1}{p(x,\theta)}\frac{\partial p(x,\theta)}{\partial \theta_{i}}\frac{\partial p(x,\theta)}{\partial \theta_{j}}dx,
	\end{aligned} 
	\label{eq:12}
\end{equation}
where the scaling number 4 stems from the Fisher information matrix which we will introduce next. 

The manifold $(S^{1/2}, I)$ of square roots of densities, together with the Riemannian metric $I(\theta) = (I_{ij}(\theta))_{i,j=1}^{n}$ defined above, is called a \textbf{statistical manifold}. 

\subsubsection{Fisher Information Matrix}
According to the above construction of a statistical manifold, we have naturally introduced the \textbf{Fisher Information Matrix} $I$ as the Riemannian metric, with each entries of $I$ defined in \eqref{eq:12}. An equivalent but more common expression of the Fisher information matrix is given by
\begin{equation}
	\begin{aligned} 
	I_{ij}(\theta) &= \mathbb{E}_{p(x,\theta)}\left[\frac{\partial \log p(x,\theta)}{\partial \theta_{i}}\frac{\partial \log p(x,\theta)}{\partial \theta_{j}}\right] \\
	&= - \mathbb{E}_{p(x,\theta)}\left[\frac{\partial^{2}\log p(x,\theta)}{\partial \theta_{i}\partial \theta_{j}}\right],
	\end{aligned} 
	\label{eq:13}
\end{equation}
and the equivalence between \eqref{eq:12} and \eqref{eq:13} can be derived from the integration-by-part formula.

Another important insight is that the equivalent definition \eqref{eq:13} of Fisher Information Matrix is the Hessian matrix of  K-L divergence:
\begin{equation}
	\left. \nabla_{\theta'}^{2} D_{KL}(p(x,\theta)\| p(x,\theta')) \right|_{\theta' = \theta}= I(\theta),
\end{equation}

The K-L divergence is invariant under reparameterization, 
as it quantifies a difference between probability distributions independently 
of their parametric representation. Consequently, $I_{ij}(\theta)$ obtained from its second-order local expansion inherits this invariance, which guarantees that the Fisher information matrix is a covariant tensor on the statistical manifold, and thus a candidate of Riemannian metric.

Notice that the gradient of K-L divergence is zero. Therefore, Fisher information matrix itself describes the local behavior of the K-L divergence, 
\begin{equation}
	D_{KL}(p(x,\theta)||p(x,\theta+\delta)) = \frac{1}{2}\delta^{\top}I(\theta)\delta + O(|\delta|^{2}),
\end{equation}
which approximates the distance between two probability distributions in a small neighborhood on a given statistical manifold. This observation provides a principled foundation for natural gradient methods, interpretable as steepest descent directions under information-theoretic constraints, leading to parameterization-invariant, stable, and well-conditioned optimization algorithms.

A natural question at this point is why the Fisher information metric is regarded as the \emph{canonical} Riemannian metric on statistical manifolds, instead of being one choice among many. The key justification is provided by Chentsov's theorem also spelled \v{C}encov's theorem, which characterizes Fisher information as essentially the unique Riemannian metric compatible with the statistical notion of information loss under data processing \cite{chentsov1982statistical,AmariNagaoka2000}.

To state the idea, consider a \emph{Markov morphism} (stochastic map) $\mathscr{T}$ that transforms an observation $x$ into $z$ via a Markov kernel $\mathscr{T}(z|x)$. This induces a push-forward of distributions, $p \mapsto \mathscr{T}p$, and represents coarse-graining or randomization of data. Since such a transformation cannot increase the distinguishability of probability distributions, it is natural to require that any statistical distance and its infinitesimal counterpart, a Riemannian metric, should be monotone under $\mathscr{T}$, i.e., it should contract under Markov morphisms.

\begin{theorem}
On the manifold of strictly positive probability distributions over a finite sample space, any smooth Riemannian metric that is monotone under all Markov morphisms is a constant multiple of the Fisher information metric. Equivalently, up to an overall positive scaling factor, the Fisher information metric is the unique Riemannian metric on the statistical manifold that is invariant under sufficient statistics and satisfies the information-processing (data-processing) principle \cite{chentsov1982statistical,AmariNagaoka2000}.
\end{theorem}

This result explains in a precise sense why the Fisher information matrix provides a natural geometry for statistical models: it is the only metric (up to scale) that respects the fundamental statistical operation of mapping data through a stochastic transformation. Consequently, optimization methods based on this geometry---notably natural gradient descent---inherit parameterization-invariance and are aligned with information-theoretic notions of distinguishability.

\subsubsection{Natural Gradient Descent}
In classical optimization, gradient descent methods rely on the Euclidean geometry of the parameter space. However, when the optimization object is a parameterized probability distribution on a statistical manifold (e.g., the loss function is treated as a function of distribution parameters), the Euclidean gradient may be an inefficient or misleading direction of steepest descent, because it does not account for the intrinsic geometry of the parameter space.

The \textbf{natural gradient} corrects this by using the Riemannian metric (Fisher information matrix) to define the steepest descent direction on the statistical manifold. Given a smooth function $f: \Theta \to \mathbb{R}$ to be minimized, the natural gradient $\tilde{\nabla} f(\theta)$ at point $\theta$ is defined as
\begin{equation}
	\tilde{\nabla} f(\theta) = I^{-1}(\theta) \nabla f(\theta),
	\label{eq:20}
\end{equation}
where $\nabla f(\theta)$ is the ordinary Euclidean gradient (a column vector), and $I^{-1}(\theta)$ is the inverse of the Fisher information matrix at $\theta$.

In fact, the natural gradient $\tilde{\nabla}f(\theta)$ defined in \eqref{eq:20} is just the \textit{gradient }on the statistical manifold $(S^{1/2},I)$ and the natural gradient descent update rule is then given by
\begin{equation}
	\theta_{t+1} = \theta_t - \eta_t \, I^{-1}(\theta_t) \nabla f(\theta_t),
\end{equation}
where $\eta_t > 0$ is the learning rate. In Bayesian filtering and machine learning, natural gradient descent often leads to faster and more stable convergence when optimizing over probability distributions.

Next, we will further illustrate the concepts of information geometry by a specific example: the exponential family. 

\subsubsection{Example: Exponential Distribution Family}
A probability density $p(x,\theta)$ belongs to the \textbf{exponential family} if it can be written in the form
\begin{equation}
	p(x,\theta) = \exp\left\{  \sum_{i=1}^m \theta_i F_i(x) - \psi(\theta) \right\},
	\label{eq:22}
\end{equation}
where $\theta = (\theta_1,\dots,\theta_m) \in \Theta\subset \mathbb{R}^{m}$ are the parameters, $F_i(x)$ are the sufficient statistics,  and $\psi(\theta)$ is the log‑partition function ensuring normalization:
\begin{equation}
	\psi(\theta) = \log \int \exp\left\{  \sum_{i=1}^m \theta_i F_i(x) \right\} dx.
\end{equation}
The most commonly-used Gaussian distribution family can also be viewed as an exponential distribution family. 

For the exponential family, the Fisher information matrix has a particularly simple form, which is stated as the following Theorem.
\begin{theorem} 
	\label{thm:1}
	For an exponential family 
	\begin{equation}
		S = \left\{p(x,\theta) =\exp\left\{  \sum_{i=1}^m \theta_i F_i(x) - \psi(\theta) \right\}:\ \theta\in \Theta \right\},
	\end{equation}
	the Fisher Information Matrix (or the Riemannian metric defined by \eqref{eq:12}) $I$ can be computed as
\begin{equation}
	I_{ij}(\theta) = \frac{\partial^2 \psi(\theta)}{\partial \theta_i \partial \theta_j}.
\end{equation}
That is, the metric is the Hessian of the log‑partition function. 
\end{theorem} 

\begin{remark}
	In the following proof of Theorem \ref{thm:1}, the elegant mathematical properties of the exponential distribution family are revealed through the derivations. The fact that $\int p(x,\theta) dx = 1$ is repeatedly used which makes the expressions more concise. For readers who are not familiar with the concept of information geometry, it is encouraged to derive this proof line-by-line, which will provide valuable insight into how the structural properties of the exponential family facilitate analytical tractability.
\end{remark}

\begin{proof} 
For a probability density function $p(x,\theta)$ defined by \eqref{eq:22}, we have
\begin{equation}
	\frac{\partial p}{\partial \theta_{i}}(x,\theta) = p(x,\theta)\left(F_{i}(x) - \frac{\partial \psi}{\partial \theta_{i}}(\theta)\right).
	\label{eq:25}
\end{equation}
Take integrals for both sides of \eqref{eq:25} and notice the fact that $\int p(x,\theta)dx = 1$. The gradient of $\psi(\theta)$ can be computed as follows:
\begin{equation}\nonumber
	0 = \frac{\partial}{\partial \theta_{i}}\int p(x,\theta)dx = \int p(x,\theta)\left(F_{i}(x) - \frac{\partial \psi}{\partial \theta_{i}}(\theta)\right) dx,
\end{equation}
and 
\begin{equation}\nonumber
	\frac{\partial \psi}{\partial \theta_{i}}(\theta) = \int F_{i}(x)p(x,\theta)dx.
	\label{eq:28}
\end{equation}
Therefore,
\begin{equation}\nonumber
	\begin{aligned} 
	\frac{\partial^{2}\psi}{\partial \theta_{i}\partial \theta_{j}}(\theta) &= \int F_{i}(x)\frac{\partial p(x,\theta)}{\partial \theta_{j}}dx. \\
	&= \int \frac{1}{p(x,\theta)}\frac{\partial p(x,\theta)}{\partial \theta_{j}}(p(x,\theta)F_{i}(x))dx.
	\end{aligned} 
\end{equation}
In the meanwhile, according to \eqref{eq:25},
\begin{equation}
	p(x,\theta)F_{i}(x) = \frac{\partial p(x,\theta)}{\partial \theta_{i}} + p(x,\theta)\frac{\partial \psi}{\partial \theta_{i}}(\theta),
\end{equation}
and thus,
\begin{equation}\nonumber
	\begin{aligned}
		\frac{\partial^{2}\psi}{\partial \theta_{i}\partial \theta_{j}}(\theta) = &  \int \frac{1}{p(x,\theta)}\frac{\partial p(x,\theta)}{\partial \theta_{j}} \frac{\partial p(x,\theta)}{\partial \theta_{i}}dx \\
		&- \frac{\partial\psi}{\partial \theta_{i}}(\theta)\frac{\partial}{\partial \theta_{j}}\int p(x,\theta)dx \\
		=&\int \frac{1}{p(x,\theta)}\frac{\partial p(x,\theta)}{\partial \theta_{j}} \frac{\partial p(x,\theta)}{\partial \theta_{i}}dx = I_{ij}(\theta).
	\end{aligned}
\end{equation}
\end{proof} 

From \eqref{eq:28} in the proof of Theorem \ref{thm:1}, another set of parameters $\varrho = (\varrho_{1},\cdots,\varrho_{m})\in\mathbb{R}^{m}$ for the exponential family can be defined by the coordinate transformation
\begin{equation}\nonumber
	\varrho_{i} := \frac{\partial \psi}{\partial \theta_{i}}(\theta) = \mathbb{E}_{p(x,\theta)}[F_{i}(x)],\ i=1,\cdots,m,
\end{equation} 
which are called the \textbf{expectation parameters}. 

The transformation matrix between $\varrho$ and $\theta$:
\begin{equation}\nonumber
	\frac{\partial \varrho}{\partial \theta} = \left[\frac{\partial^{2}\psi}{\partial \theta_{i}\partial \theta_{j}}\right]_{ij} = I(\theta),
\end{equation}
is just the Fisher information matrix. In this way, the Fisher information matrix with respect to the expectation parameters $\varrho$ can be written as
\begin{equation}\nonumber
	\begin{aligned} 
	\tilde{I}(\varrho) &= E_{p(x,\varrho)}\left[\nabla_{\varrho} \log p(x,\varrho)\nabla_{\varrho}\log p(x,\varrho)^{\top}\right] \\
	&= E_{p(x,\varrho)}\left[\left(\frac{\partial\theta}{\partial \varrho}\right)^{\top}\nabla_{\theta}\log p(x,\theta)\nabla_{\theta}\log p(x,\theta)\left(\frac{\partial\theta}{\partial \varrho}\right)\right] \\
	&=\left(\frac{\partial\theta}{\partial \varrho}\right)^{\top}E_{p(x,\varrho)}\left[\nabla_{\theta}\log p(x,\theta)\nabla_{\theta}\log p(x,\theta)\right]\left(\frac{\partial\theta}{\partial \varrho}\right) \\
	&= I(\theta)^{-1} I(\theta) I(\theta)^{-1} = I(\theta)^{-1},
	\end{aligned} 
\end{equation}
which is the inverse of the Fisher information matrix with respect to the original parameters $\theta$. The relationship between the Fisher information matrix with respect to $\theta$ and $\varrho$ is useful in the derivation of the  natural gradient Gaussian filter introduced in the next section, because the Gaussian distribution (as an exponential distribution family) is parameterized by the expectation parameters (i.e., the mean and covariance matrix).

%
%
%

\section{Natural Gradient Filtering}\label{sec.ngf}

\begin{figure}
    \centering
    \includegraphics[width=0.99\linewidth]{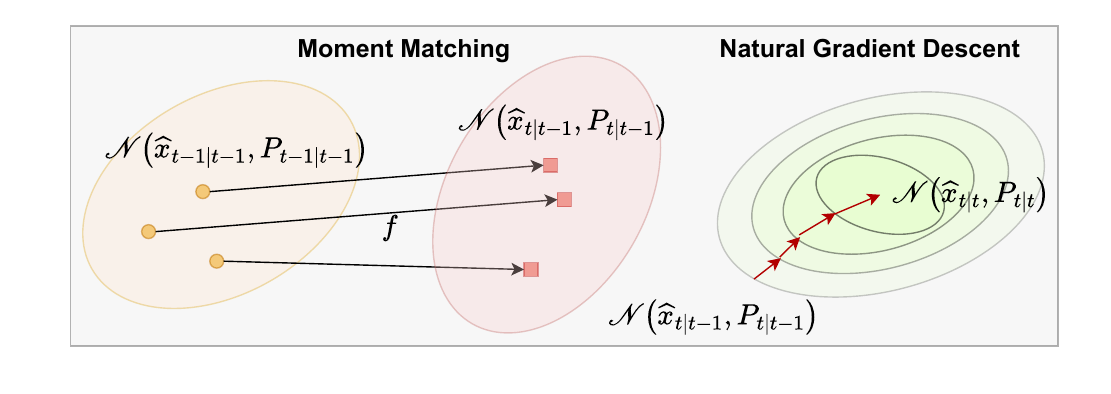}
    \caption{Illustration of NANO filter}
    \label{fig.NANO}
\end{figure}

\subsection{An Optimization Perspective on Gaussian Filtering}\label{sec.opt_gaussian_filtering}

The variational characterizations \eqref{eq.BF_prediction_optimization_rewrite}--\eqref{eq.BF_update_optimization_rewrite} interpret Bayesian filtering as two optimization problems over the infinite-dimensional space of probability densities. While conceptually clean, these problems are generally intractable because the optimizer is an arbitrary density function. 
Gaussian filtering can be viewed as a principled finite-dimensional relaxation: one restricts the candidate density to a parameterized family—most commonly the multivariate Gaussian family—thereby converting the original variational problems into optimization problems over a mean--covariance pair.

Formally, define the Gaussian family
\[
\mathcal{Q}_{\mathrm{G}}
\triangleq
\Bigl\{\,q(x)=\mathcal{N}(x;\hat{x},P)\ :\ \hat{x}\in\mathbb{R}^n,\ P\in\mathbb{S}^n_{++}\Bigr\},
\]
where $\mathbb{S}^n_{++}$ denotes the set of symmetric positive definite matrices.
A Gaussian filter may then be interpreted as constructing a computable trajectory
\[
t\ \mapsto\ q_t(\cdot)=\mathcal{N}(\cdot;\hat{x}_{t|t},P_{t|t})\in\mathcal{Q}_{\mathrm{G}},
\]
such that $q_t$ remains ``close'' to the true filtering posterior $p(x_t|y_{1:t})$ in an information-theoretic sense.
From the optimization viewpoint, this ``closeness'' is induced by KL divergences embedded in the variational formulations, which leads naturally to projection-like and proximal-like interpretations of the prediction and update steps.

Restricting the prediction variational problem \eqref{eq.BF_prediction_optimization_rewrite} to the Gaussian family yields
\begin{equation}\label{eq.gaussian_prediction_opt_review}
\begin{aligned}
&(\hat{x}_{t|t-1},P_{t|t-1})
\\
=&
\argmax_{\hat{x},P}\ 
\Exp_{p(x_{t-1}|y_{1:t-1})}\Exp_{p(x_t|x_{t-1})}
\Bigl[\log\mathcal{N}(x_t;\hat{x},P)\Bigr].
\end{aligned}
\end{equation}
Note that since
\[
\Exp_{p_{t|t-1}}[\log\mathcal{N}(x;\hat{x},P)]
=
-\ D_{\mathrm{KL}}\!\bigl(p_{t|t-1}\,\|\,\mathcal{N}(\hat{x},P)\bigr)\ +\ \text{const},
\]
the maximizer of \eqref{eq.gaussian_prediction_opt_review} is the forward-KL projection of the true prior onto $\mathcal{Q}_{\mathrm{G}}$. 
This is a key conceptual point: the Gaussian prediction step does not simply ``assume Gaussianity''; rather, it selects the Gaussian that is information-theoretically closest to the true prior under $D_{\mathrm{KL}}(p\|q)$.

Operationally, the projection admits a particularly simple characterization: the optimal Gaussian matches the first two moments of $p_{t|t-1}$.
This recovers the classical ``moment propagation'' viewpoint of many Gaussian filters.

\begin{lemma}[\cite{cao2024nonlinear}]\label{lem:stationary_pred_moment} Let $p(x)$ be a density on $\mathbb{R}^n$ with finite second moments. Consider the maximum expected Gaussian log-likelihood problem
\[
(\mu^\star,\Sigma^\star)
=
\argmax_{\mu,\Sigma\in\mathbb{S}^n_{++}}
\ \Exp_{p(x)}\bigl[\log\mathcal{N}(x;\mu,\Sigma)\bigr].
\]
Any stationary point satisfies
\begin{subequations}\nonumber
\begin{align}
\mu^\star &= \Exp_{p(x)}[x],\\
\Sigma^\star &= \Exp_{p(x)}\bigl[(x-\mu^\star)(x-\mu^\star)^\top\bigr]
=\Exp_{p(x)}[xx^\top]-\mu^\star{\mu^\star}^\top.
\end{align}
\end{subequations}
\end{lemma}

Applying Lemma \ref{lem:stationary_pred_moment} with $p(x)=p(x_t|y_{1:t-1})$ yields the Gaussian prediction equations
\begin{subequations}\label{eq.gaussian_prediction_moment}
\begin{align}
\hat{x}_{t|t-1} &= \Exp_{p(x_t|y_{1:t-1})}[x_t],\\
P_{t|t-1} &= \Exp_{p(x_t|y_{1:t-1})}[x_tx_t^\top]-\hat{x}_{t|t-1}\hat{x}_{t|t-1}^\top.
\end{align}
\end{subequations}
For the state-space model \eqref{eq.SSM} with zero-mean process noise, \eqref{eq.gaussian_prediction_moment} reduces to computing Gaussian expectations of nonlinear functions, e.g.,
\[
\hat{x}_{t|t-1} = \Exp_{\mathcal{N}(x_{t-1};\hat{x}_{t-1|t-1},P_{t-1|t-1})}\!\bigl[f(x_{t-1})\bigr],
\]
together with the corresponding second-moment expression for $P_{t|t-1}$.
This is exactly the computational bottleneck addressed by unscented/cubature/quadrature moment transforms: they approximate integrals of the form $\Exp_{\mathcal{N}}[\varphi(x)]$ without requiring linearization of $f$.

Restricting the update variational problem \eqref{eq.BF_update_optimization_rewrite} to $\mathcal{Q}_{\mathrm{G}}$ gives
\begin{equation}\label{eq.gaussian_update_opt_review}
\begin{aligned}
&(\hat{x}_{t|t},P_{t|t})
=
\argmin_{\hat{x},P\in\mathbb{S}^n_{++}}
\Bigl\{
\Exp_{\mathcal{N}(x_t;\hat{x},P)}\bigl[\ell(x_t,y_t)\bigr]
\\
&+
D_{\mathrm{KL}}\!\bigl(\mathcal{N}(x_t;\hat{x},P)\ \|\ \mathcal{N}(x_t;\hat{x}_{t|t-1},P_{t|t-1})\bigr)
\Bigr\},
\end{aligned}
\end{equation}
where $\ell(x_t,y_t)\triangleq-\log p(y_t|x_t)$ is the measurement-dependent loss.
Problem \eqref{eq.gaussian_update_opt_review} makes the algorithmic role of the prior explicit: the KL term acts as an information-theoretic regularizer that keeps the updated Gaussian close to the predicted Gaussian, while the expected loss term encourages consistency with the new measurement.
In this sense, the update step is a distributional proximal step: it trades off data fit against deviation from the prediction distribution.

Expanding the Gaussian--Gaussian KL divergence yields an explicit finite-dimensional objective:
\begin{equation}\label{eq.explicit_J_long}
\begin{aligned}
J(\hat{x},P)
=&\ \Exp_{\mathcal{N}(x;\hat{x},P)}\bigl[\ell(x,y_t)\bigr]\\
&+\frac{1}{2}(\hat{x}-\hat{x}_{t|t-1})^\top P_{t|t-1}^{-1}(\hat{x}-\hat{x}_{t|t-1})
\\
&+\frac{1}{2}\mathrm{Tr}\!\left(P_{t|t-1}^{-1}P\right)
-\frac{1}{2}\log\frac{|P|}{|P_{t|t-1}|}-\frac{n}{2}.
\end{aligned}
\end{equation}
Compared with prediction, the update step is harder for a fundamental reason: the expectations in $J(\hat{x},P)$ are taken under the \emph{unknown} posterior Gaussian itself. Therefore, even when $\ell$ is smooth, the optimality conditions typically become implicit fixed-point equations coupling $\hat{x}_{t|t}$ and $P_{t|t}$.

To state the stationary conditions cleanly, we adopt the standard smoothness assumption.

\begin{assumption}\label{assump:smooth_loss_review}
For fixed $y_t$, the loss $\ell(\cdot,y_t)$ is twice continuously differentiable and the required expectations under Gaussian measures exist.
\end{assumption}

Under Assumption \ref{assump:smooth_loss_review}, one may use Stein-type identities (Bonnet's \cite{bonnet1964transformations}
and Price's \cite{price1958useful}
theorems) to differentiate Gaussian expectations with respect to their mean and covariance parameters. This produces compact expressions for the gradients of \eqref{eq.explicit_J_long}, and hence its stationary conditions.

\begin{lemma}\label{lem:stationary_update_fixedpoint}
Under Assumption \ref{assump:smooth_loss_review}, any stationary point $(\hat{x}_{t|t},P_{t|t})$ of \eqref{eq.gaussian_update_opt_review} satisfies the coupled equations
\begin{subequations}\label{eq.stationary_update_fixedpoint}
\begin{align}
\hat{x}_{t|t}
&=
\hat{x}_{t|t-1}
-
P_{t|t-1}\,
\Exp_{\mathcal{N}(x_t;\hat{x}_{t|t},P_{t|t})}\!\bigl[\nabla_{x_t}\ell(x_t,y_t)\bigr],
\label{eq.stationary_update_mean}\\
P_{t|t}^{-1}
&=
P_{t|t-1}^{-1}
+
\Exp_{\mathcal{N}(x_t;\hat{x}_{t|t},P_{t|t})}\!\bigl[\nabla_{x_t}^2\ell(x_t,y_t)\bigr].
\label{eq.stationary_update_cov}
\end{align}
\end{subequations}
\end{lemma}

Equations \eqref{eq.stationary_update_fixedpoint} are generally not solvable in closed form because both expectations are evaluated under $\mathcal{N}(\hat{x}_{t|t},P_{t|t})$, which is precisely the unknown to be determined.
In other words, the Gaussian posterior is characterized as a fixed point of a map that depends on the score $\nabla\ell$ and the curvature $\nabla^2\ell$ averaged under the posterior approximation.
This explains why many nonlinear Gaussian filters resort to further approximations: one must either (a) approximate the expectations, (b) approximate $\ell$ by a quadratic surrogate, or (c) solve the fixed point iteratively.

\subsection{Natural Gradient Gaussian Approximation}\label{sec:NANO}

The prediction step in \eqref{eq.gaussian_prediction_opt_review} admits an exact solution within the Gaussian family: the forward-KL projection is achieved by moment matching (Lemma~\ref{lem:stationary_pred_moment}). 
In contrast, the update step \eqref{eq.gaussian_update_opt_review} is substantially more challenging, because the optimality conditions \eqref{eq.stationary_update_fixedpoint} are implicit and couple $(\hat{x}_{t|t},P_{t|t})$ through expectations taken under the \emph{unknown} posterior Gaussian. 
Existing Gaussian filters often resolve this difficulty by linearization (or local quadratic surrogates), which introduces approximation errors in locating the stationary point.
Motivated by information-geometric optimization, we instead seek to solve \eqref{eq.gaussian_update_opt_review} by an iterative \emph{natural gradient} scheme on the Gaussian statistical manifold.

We stack the mean and the \emph{information-form} covariance parameter into a single vector
\begin{equation}\label{eq:NANO_stack_def}
v \triangleq 
\begin{bmatrix}
\hat{x}_{t|t}\\
\mathrm{vec}(P_{t|t}^{-1})
\end{bmatrix},
\; 
\frac{\partial}{\partial v}J(\hat{x}_{t|t},P_{t|t})
=
\begin{bmatrix}
\frac{\partial}{\partial \hat{x}_{t|t}}J(\hat{x}_{t|t},P_{t|t})\\[2pt]
\mathrm{vec}\!\left(\frac{\partial}{\partial P_{t|t}^{-1}}J(\hat{x}_{t|t},P_{t|t})\right)
\end{bmatrix},
\end{equation}
where we optimize over $P_{t|t}^{-1}$ rather than $P_{t|t}$. 
This choice is aligned with the information-filter representation and frequently yields simpler expressions for update rules.

Let $v^{(i)}$ denote the parameters at iteration $i$, and define
\[
\delta v \triangleq v^{(i+1)}-v^{(i)}
=
\begin{bmatrix}
\hat{x}_{t|t}^{(i+1)}-\hat{x}_{t|t}^{(i)}\\[2pt]
\mathrm{vec}\!\bigl((P_{t|t}^{-1})^{(i+1)}-(P_{t|t}^{-1})^{(i)}\bigr)
\end{bmatrix}.
\]
The natural gradient update is
\begin{equation}\label{eq:NANO_natgrad_generic}
\delta v
=
-
\left[
\mathcal{F}_v^{-1}\,
\frac{\partial}{\partial v}J(\hat{x}_{t|t},P_{t|t})
\right]_{v=v^{(i)}},
\end{equation}
where $\mathcal{F}_v$ is the Fisher information matrix of $\mathcal{N}(x_t;\hat{x}_{t|t},P_{t|t})$ under the chosen parameterization.

\begin{proposition}\label{prop:NANO_fisher}
For $\mathcal{N}(x_t;\hat{x}_{t|t},P_{t|t})$ parameterized by $v=[\hat{x}_{t|t}^\top,\ \mathrm{vec}(P_{t|t}^{-1})^\top]^\top$, the inverse Fisher information matrix is
\begin{equation}\label{eq:NANO_fisher_inv}
\mathcal{F}_v^{-1}
=
\begin{bmatrix}
P_{t|t} & 0\\
0 & 2\,(P_{t|t}^{-1}\otimes P_{t|t}^{-1})
\end{bmatrix},
\end{equation}
where $\otimes$ denotes the Kronecker product.
\end{proposition}

Combining \eqref{eq:NANO_natgrad_generic}--\eqref{eq:NANO_fisher_inv} and rewriting the vectorized expressions back in matrix form yields the iterative scheme

\begin{subequations}\label{eq:NANO_iter_updates_general}
\small{
\begin{align}
(P_{t|t}^{-1})^{(i+1)}
&=
(P_{t|t}^{-1})^{(i)}
-
2\,(P_{t|t}^{-1})^{(i)}
\left.\frac{\partial}{\partial P_{t|t}^{-1}}J(\hat{x}_{t|t},P_{t|t})\right|_{v^{(i)}}
(P_{t|t}^{-1})^{(i)},\\
\hat{x}_{t|t}^{(i+1)}
&=
\hat{x}_{t|t}^{(i)}
-
P_{t|t}^{(i+1)}
\left.\frac{\partial}{\partial \hat{x}_{t|t}}J(\hat{x}_{t|t},P_{t|t})\right|_{v^{(i)}}.
\end{align}
}
\end{subequations}

Using the explicit objective \eqref{eq.explicit_J_long} and the standard identities for differentiating Gaussian expectations (Bonnet's and Price's theorems), one obtains the following natural-gradient iteration:
\begin{subequations}\label{eq:NANO_natgrad_update_1}
\begin{align}
(P_{t|t}^{-1})^{(i+1)}
=&
P_{t|t-1}^{-1}
+
\Exp_{\mathcal{N}(x_t;\hat{x}_{t|t}^{(i)},P_{t|t}^{(i)})}
\!\left[\nabla_{x_t}^2\ell(x_t,y_t)\right],\label{eq:NANO_natgrad_cov}\\
\hat{x}_{t|t}^{(i+1)}
=&
\hat{x}_{t|t}^{(i)}
-
P_{t|t}^{(i+1)}
\Exp_{\mathcal{N}(x_t;\hat{x}_{t|t}^{(i)},P_{t|t}^{(i)})}
\!\left[\nabla_{x_t}\ell(x_t,y_t)\right]
\\
&-
P_{t|t}^{(i+1)}P_{t|t-1}^{-1}\bigl(\hat{x}_{t|t}^{(i)}-\hat{x}_{t|t-1}\bigr).
\label{eq:NANO_natgrad_mean}
\end{align}
\end{subequations}
Unlike one-shot linearization updates, \eqref{eq:NANO_natgrad_update_1} directly targets the stationary point of \eqref{eq.gaussian_update_opt_review} through a principled descent direction defined by the Fisher metric.

A practical obstacle in \eqref{eq:NANO_natgrad_update_1} is the need to compute derivatives of $\ell(x_t,y_t)$.
To avoid this, we invoke Stein-type identities to rewrite the required expectations in terms of \emph{loss-weighted moments} under the current Gaussian iterate:
\begin{subequations}\label{eq:NANO_derivative_free_id}
\begin{align}
\Exp_{\mathcal{N}}\!\left[\nabla_{x_t}\ell(x_t,y_t)\right]
&=
(P_{t|t}^{-1})^{(i)}
\Exp_{\mathcal{N}}
\!\left[e_{t|t}^{(i)}\,\ell(x_t,y_t)\right],\\
\begin{split}
\Exp_{\mathcal{N}}\!\left[\nabla_{x_t}^2\ell(x_t,y_t)\right]
&=
(P_{t|t}^{-1})^{(i)}
\Exp_{\mathcal{N}}
\!\left[e_{t|t}^{(i)} (e_{t|t}^{(i)})^\top\,\ell(x_t,y_t)\right]\\
&\quad-(P_{t|t}^{-1})^{(i)}\Exp_{\mathcal{N}}\!\left[\ell(x_t,y_t)\right],
\end{split}
\end{align}
\end{subequations}
where $\mathcal{N}$ is shorthand for $\mathcal{N}(x_t;\hat{x}_{t|t}^{(i)},P_{t|t}^{(i)})$ and $e_{t|t}^{(i)} = x_t-\hat{x}_{t|t}^{(i)}$.
Substituting \eqref{eq:NANO_derivative_free_id} into \eqref{eq:NANO_natgrad_update_1} yields a derivative-free update:

\begin{subequations}\label{eq:NANO_derivative_free_update}

\begin{align}
(P_{t|t}^{-1})^{(i+1)}
=&\ 
(P_{t|t}^{-1})^{(i)}
\Exp_{\mathcal{N}}\!\left[e_{t|t}^{(i)} \bigl(e_{t|t}^{(i)}\bigr)^\top\,\ell(x_t,y_t)\right]
(P_{t|t}^{-1})^{(i)}
\nonumber\\
&\ 
-
(P_{t|t}^{-1})^{(i)}\Exp_{\mathcal{N}}\!\left[\ell(x_t,y_t)\right] + P_{t|t-1}^{-1},
\\
\hat{x}_{t|t}^{(i+1)}
=&\ 
\hat{x}_{t|t}^{(i)}
-
P_{t|t}^{(i+1)}(P_{t|t}^{-1})^{(i)}
\Exp_{\mathcal{N}}\!\left[\bigl(e_{t|t}^{(i)}\bigr)\,\ell(x_t,y_t)\right] \nonumber
\\
&-
P_{t|t}^{(i+1)}P_{t|t-1}^{-1}\bigl(\hat{x}_{t|t}^{(i)}-\hat{x}_{t|t-1}\bigr).
\end{align}
\end{subequations}

The NANO filter algorithm is summarized in Algorithm \ref{alg.1} and Fig. \ref{fig.NANO} respectively.

\begin{algorithm}[t]
\caption{NANO Filter}\label{alg.1}
\begin{algorithmic} 
    \State \textbf{Input:} Stopping thresohold $\gamma$
    \State \textbf{Initialization:} State estimate $\hat{x}_{0|0}$ and covariance $P_{0|0}$
    \For{each time step $t$}
        \State \textbf{Predict:}
        \State Calculate predicted state mean $\hat{x}_{t|t-1}$ and covariance $P_{t|t-1}$ using \eqref{eq.gaussian_prediction_moment}
        \State \textbf{Update:}
        \State Obtain the noisy measurement $y_t$ 
        \State  Initialize the state estimate $\hat{x}_{t}^{(0)}$ and covariance $P_{t}^{(0)}$
        \For{each iteration number $i$} 
        \State {Update state estimate and covariance using  \eqref{eq:NANO_natgrad_update_1}}
        \EndFor
        \State 
        $
        \hat{x}_{t|t} = \hat{x}_t^{(i)}, P_{t|t} = P_t^{(i)}
        $
    \EndFor
\end{algorithmic}
\end{algorithm}

\subsection{Connection to Kalman Filter}\label{sec:connection_kf}

We now show that, for linear Gaussian measurement models, one iteration of the natural-gradient Gaussian update
\eqref{eq:NANO_natgrad_update_1} recovers exactly the Kalman filter (KF) measurement update.

Consider the standard linear observation model
\[
p(y_t\mid x_t)=\mathcal{N}(y_t; Cx_t, R),
\]
whose negative log-likelihood (up to an additive constant independent of $x_t$) is
\[
\ell(x_t,y_t)=\tfrac12(y_t-Cx_t)^\top R^{-1}(y_t-Cx_t).
\]
Its first- and second-order derivatives are
\begin{align*}
\nabla_{x_t}\ell(x_t,y_t)&=C^\top R^{-1}(Cx_t-y_t),\\
\nabla^2_{x_t}\ell(x_t,y_t)&=C^\top R^{-1}C.
\end{align*}
Importantly, $\nabla^2_{x_t}\ell$ is constant in $x_t$, while $\nabla_{x_t}\ell$ is affine in $x_t$.
 
Substituting $\nabla^2_{x_t}\ell(x_t,y_t)=C^\top R^{-1}C$ into the NANO covariance update
\eqref{eq:NANO_natgrad_cov} shows that the expectation is trivial under any Gaussian iterate:
\[
\Exp_{\mathcal{N}(x_t;\hat{x}_{t|t}^{(0)},P_{t|t}^{(0)})}\!\left[\nabla_{x_t}^2\ell(x_t,y_t)\right]
=
C^\top R^{-1}C.
\]
Therefore, after one iteration,
\begin{equation}\label{eq:kf_equiv_info_cov}
(P_{t|t}^{-1})^{(1)} = P_{t|t-1}^{-1} + C^\top R^{-1}C.
\end{equation}
Equation \eqref{eq:kf_equiv_info_cov} is exactly the information-form KF covariance update,
hence $P_{t|t}^{(1)}=P_{t|t}$ regardless of the initialization $P_{t|t}^{(0)}$.
 
Next, since $\nabla_{x_t}\ell(x_t,y_t)=C^\top R^{-1}(Cx_t-y_t)$ is affine in $x_t$,
its expectation under $\mathcal{N}(x_t;\hat{x}_{t|t}^{(0)},P_{t|t}^{(0)})$ reduces to evaluating at the mean:
\[
\Exp_{\mathcal{N}(x_t;\hat{x}_{t|t}^{(0)},P_{t|t}^{(0)})}\!\left[\nabla_{x_t}\ell(x_t,y_t)\right]
=
C^\top R^{-1}(C\hat{x}_{t|t}^{(0)}-y_t).
\]
Plugging this and $P_{t|t}^{(1)}=P_{t|t}$ into the NANO mean update \eqref{eq:NANO_natgrad_mean} (with $i=0$) yields
\begin{equation}\label{eq:kf_equiv_mean_step1}
\hat{x}_{t|t}^{(1)}
=
\hat{x}_{t|t}^{(0)}
+
P_{t|t}C^\top R^{-1}(y_t-C\hat{x}_{t|t}^{(0)})
-
P_{t|t}P_{t|t-1}^{-1}(\hat{x}_{t|t}^{(0)}-\hat{x}_{t|t-1}).
\end{equation}
Using the information-form identity implied by \eqref{eq:kf_equiv_info_cov},
\[
P_{t|t}P_{t|t-1}^{-1}
=
P_{t|t}\bigl(P_{t|t}^{-1}-C^\top R^{-1}C\bigr)
=
I - P_{t|t}C^\top R^{-1}C,
\]
one can simplify \eqref{eq:kf_equiv_mean_step1} and eliminate the dependence on $\hat{x}_{t|t}^{(0)}$, obtaining
\begin{equation}\label{eq:kf_equiv_mean_step2}
\hat{x}_{t|t}^{(1)}
=
\hat{x}_{t|t-1}
+
P_{t|t}C^\top R^{-1}\bigl(y_t-C\hat{x}_{t|t-1}\bigr).
\end{equation}
Finally, by the matrix inversion lemma, the factor multiplying the innovation can be rewritten as the Kalman gain:
\[
P_{t|t}C^\top R^{-1}
=
P_{t|t-1}C^\top\bigl(CP_{t|t-1}C^\top+R\bigr)^{-1}
\triangleq K_t.
\]
Substituting into \eqref{eq:kf_equiv_mean_step2} gives the standard KF mean update.

For the linear Gaussian system, the natural-gradient update
\eqref{eq:NANO_natgrad_update_1} reaches the unique stationary point of \eqref{eq.gaussian_update_opt_review} in one iteration.
Equivalently, a single NANO update (i.e., $i=0\to1$) produces exactly the KF posterior $(\hat{x}_{t|t},P_{t|t})$,
independently of the initialization $(\hat{x}_{t|t}^{(0)},P_{t|t}^{(0)})$.

\subsection{Practical Implementation}
To transition from the theoretical derivation of the NANO filter to its actual deployment in dynamical systems, several numerical aspects must be addressed. This section details the practical realization of the algorithm, focusing on two critical components: the numerical computation of Gaussian integrals in expectation calculation and the enforcement of covariance positive definiteness to ensure numerical stability during iterative updates.

\subsubsection{Expectation Calculation}
The moment matching (MM) technique is an important computational tool used in expectation calculation. It approximates the mean \(\mu^\prime\) and covariance \(\Sigma^\prime\) of a random variable \(x \sim \mathcal{N}(\mu,\Sigma)\) after passing through a nonlinear function \(f(\cdot)\), which is concisely expressed as $\{\mu', \Sigma'\} = \mathrm{MM}(\mu, \Sigma; f(\cdot))$.
Its general computational framework is
\begin{equation}
\label{eq:moment_matching}
\begin{aligned}
\nonumber
\mu^\prime &= \sum_{i=0}^{N-1} w_{m}^i\,f(\chi_i),\\
\Sigma^\prime &= \sum_{i=0}^{N-1} w_{c}^i\!\left[f(\chi_i)-\mu^\prime\right]\!\left[f(\chi_i)-\mu^\prime\right]^{\top}, 
\end{aligned}
\end{equation}
where \(\chi_i\) is the collocation point and \(w_m^i\) and \(w_c^i\) is the weighting coefficent. Depending on how collocation point and weighting coefficent are constructed, MM can be implemented in different forms, such as unscented transformation \cite{julier1995new,julier2004unscented}, Gauss–Hermite quadrature \cite{arasaratnam2007discrete} or spherical–radial cubature \cite{arasaratnam2009cubature}, etc.

Taking the unscented transformation as a representative example, the collocation points are explicitly defined as:
\begin{equation}
\nonumber
\chi_0 = \mu, \quad \chi_i = \mu \pm \sqrt{(n + \lambda) \Sigma}_i , \quad  i = 1, \dots, n,
\end{equation}
and the corresponding weighting coefficients are given by
\begin{equation}
\nonumber
\begin{aligned}
 w_m^0 &= \frac{\lambda}{n + \lambda}, \quad w_c^0 = w_m^0 + (1 - \alpha^2 + \beta)
 \\
 w_m^i &= w_c^i =  \frac{1}{2(n + \lambda)},
 \quad
 i = 1, \dots, 2n.
 \end{aligned}    
\end{equation}
where $n$ is the state dimension, $\lambda$ is a composite scaling parameter, and \( \sqrt{(n + \lambda) \Sigma} \) is the Cholesky decomposition of \( (n + \lambda) \Sigma \). The parameters \( \alpha \), \( \beta \), and \( \lambda \) control the spread and weighting of the sigma points, with common defaults being \( \alpha = 0 \), \( \beta = 1 \), and \( \lambda = 0 \) for the so-called Julier's sigma points \cite{julier1995new}.

\subsubsection{Positive Definiteness Guarantee}
During the iterative update process in \eqref{eq:NANO_natgrad_cov}, the covariance matrix is not guaranteed to remain positive definite. This stems from the structure of the covariance update rule: while the prior covariance \(P_{t|t-1}\) is strictly positive definite, the Hessian term involved in the update introduces indefinite components. Specifically, the log-likelihood Hessian takes the form
\begin{equation}
\begin{aligned}
\label{eq.hessian}
\nabla_x^2 \ell(x, y_t) = G^\top R^{-1}G - \nabla_x^2g(x)^\top R^{-1}(y_t - g(x)),
\end{aligned}
\end{equation}
where \(G = \nabla_xg(x)\) is the Jacobian of the measurement function. The first term \(G^\top R^{-1} G\) is always positive semi-definite. In contrast, the second term depends on the measurement residual \((y_t - g(x))\) and the Hessian of the measurement function \(\nabla_x^2g(x)\), which can take either positive or negative values. As a result, the full Hessian is generally indefinite. Therefore, using the iteration in \eqref{eq:NANO_natgrad_cov} may cause the covariance matrix to become non-positive definite at certain time instant, leading to the divergence of the NANO filter. To address this problem, we introduce two methods that guarantee the positive definiteness of the covariance updates \cite{Zhang2026NaturalGG}.

\textbf{Hessian Approximation.}
The first approach is to directly approximate the Hessian \eqref{eq.hessian} as a positive semi-definite matrix. We define the normalized measurement residual as \(r(x,y_t)= R^{-1/2}(y_t - g(x))\). With this representation, the log-likelihood can be transformed into a least-squares form as \(\ell(x,y_t) = \frac{1}{2}\|r(x,y_t)\|_2^2 \). Therefore, we have \(\ell(x+\Delta x,y_t) = \frac{1}{2}\|r(x+\Delta x,y_t)\|_2^2 \), and then performing a Taylor expansion on both sides of the equation, we have
\begin{equation}\nonumber
\begin{aligned}
\ell(x+\Delta x,y_t) \approx& \ell(x,y_t) + \frac{\partial \ell}{\partial x}^\top\Delta x + \frac{1}{2}\Delta x^\top\frac{\partial^2 \ell}{\partial x^2}\Delta x, \\
\frac{1}{2}\|r(x+\Delta x,y_t)\|_2^2 \approx& \frac{1}{2}\|r(x,y_t) + J(x)\Delta x\|_2^2 \\
=& \frac{1}{2}\|r(x,y_t)\|_2^2 + r(x,y_t)^\top J(x)\Delta x \\
& + \frac{1}{2}\Delta x^\top J(x)^\top J(x)\Delta x,
\end{aligned}
\end{equation}
where \(J(x) = {\partial r}/{\partial x} = -R^{-1/2}G\) denote the Jacobian of \(r(x,y_t)\). Meanwhile, it is easy to verify that \({\partial \ell^\top}/{\partial x} =r(x,y_t)^\top J(x) \), so in the second-order Taylor sense, the log-likelihood Hessian can be approximated by the self-adjoint product of the residual’s Jacobian as
\begin{equation}\nonumber
\frac{\partial^2 \ell(x, y_t)}{\partial x^2} \approx J(x)^\top J(x)=G^\top  R^{-1}G.
\end{equation}

This approximation ensures that the Hessian is positive semidefinite, which in turn guarantees that the covariance matrix remains positive definite throughout the entire iterative process. The proof is as follows:

\begin{proof}
For any non-zero vector \(x \in \mathbb{R}^n\), consider the quadratic form as
\begin{equation}\nonumber
\begin{aligned}
\ x^\top\frac{\partial^2 \ell(x, y_t)}{\partial x_t^2}x&\approx x^\top G^\top R^{-1} Gx, \\
&=(Gx)^\top R^{-1} (Gx).
\end{aligned}
\end{equation}
Since \(R\) is positive definite, \((G x)^\top R^{-1} (G x) \geq 0\) for all x, with equality if and only if \(G x = 0\). This confirms that the approximated Hessian is positive semi-definite. Moreover, since the MM step in \eqref{eq:moment_matching} involves summation and \(P_{t|t-1}^{-1}\) is strictly positive definite, the iteratively obtained \((P_t^{-1})^{(k+1)}\) is guaranteed to be positive definite.
\end{proof}

Note that this approximation of the Hessian actually ignores the second term of the exact Hessian \eqref{eq.hessian} and is equivalent to the Gauss–Newton method \cite{bell1993iterated}. This is reasonable because, when the normalized residual is relatively small, such as during stable iterations where the Gaussian approximation closely matches the true posterior, the second term’s contribution to the Hessian becomes insignificant compared to the dominant positive semi-definite first term.

\textbf{Cholesky Decomposition.}
Another approach exploits the fact that any positive-definite matrix admits a Cholesky decomposition.  
Before the iteration of update step at time instant $t$ begins, the inverse covariance matrix is factorized as \((P_t^{-1})^{(k)} = (\Lambda_t^{(k)})(\Lambda_t^{(k)})^\top\), where \(\Lambda_t^{(k)}\) is a lower-triangular matrix. Based on the conclusions in \cite{lin2021tractable, goudar2022gaussian}, the original covariance iteration \eqref{eq:NANO_natgrad_cov} can be directly written as an iteration for the decomposed matrix \(\Lambda_t^{(k)}\) as
\begin{equation}
\label{eq.inter_cov_2}
\Lambda^{(k+1)} = \Lambda^{(k)}\exp_m(\frac{1}{2}(\Lambda^{(k)})^{-1}(V_{xx}^{(k)}+P_{t|t-1}^{-1})(\Lambda^{(k)})^{-\top}),
\end{equation}
where \(\exp_m(\cdot)\) is the matrix exponential function and \(V_{xx}^{(k)}=\Exp_{\mathcal{N}(x_t;\hat{x}_{t|t}^{(i)},P_{t|t}^{(i)})}
\!\left[\nabla_{x_t}^2\ell(x_t,y_t)\right]\). To make the computation more tractable, we can further simplify \eqref{eq.inter_cov_2} by using a first-order approximation of the matrix exponential map, yields
\begin{equation}
\nonumber
\Lambda^{(k+1)} \approx \Lambda^{(k)} +\frac{1}{2}(V_{xx}^{(k)}+P_{t|t-1}^{-1})(\Lambda^{(k)})^{-\top}.
\end{equation}
During the iteration, we update \(\Lambda_t\) directly and reconstruct the inverse of covariance using \((P_t^{-1})^{(k+1)} =(\Lambda_t^{(k+1)})(\Lambda_t^{(k+1)})^\top+\epsilon I \), where \(\epsilon>0\) is a small tunable parameter. This factorization and reconstruction can ensure that the covariance matrix remains positive definite throughout. The proof is as follows:

\begin{proof}
For any non-zero vector \(x \in \mathbb{R}^n\), we have
\begin{equation}
\nonumber
\begin{aligned}
x^\top(\Lambda_t^{(k+1)})(\Lambda_t^{(k+1)})^\top x=\|(\Lambda_t^{(k+1)})^\top x\|^2_2 \geq 0,
\end{aligned}
\end{equation}  
so the matrix \((\Lambda_t^{(k)}) (\Lambda_t^{(k)})^\top\) is positive semi-definite. Furthermore, by adding a small positive definite matrix \(\epsilon I\), the inverse of the covariance \((P_t^{-1})^{(k+1)}\) becomes positive definite, which means the covariance \(P_t^{(k+1)}\) is positive definite.
\end{proof}

In summary, both approaches ensure that the covariance matrix remains positive definite throughout the natural gradient updates. However, the second approach involves performing Cholesky decomposition and first-order approximation at each iteration, which results in higher computational complexity and potential errors. On the other hand, the first approach offers second-order approximation accuracy and avoids redundant calculations. Therefore, the first approach is typically used as the primary method for the standard NANO filter.

\subsection{NANO-filter on Manifold}

Most of the filters discussed above are designed in Euclidean vector spaces. In practice, however, the state of an autonomous robot often evolves on smooth manifolds that admit Lie-group structure. For example, orientation lies on the special orthogonal group \(\mathrm{SO}(3)\), and pose lies on the special Euclidean group \(\mathrm{SE}(3)\)~\cite{barfoot2024state}. Applying Euclidean filters to such non-Euclidean states ignores the underlying geometry, which can lead to singularities, constraint violations, and instability~\cite{barrau2016invariant,hartley2020contact}. Motivated by this, we can design our NANO-filter to operate directly on manifolds \cite{Zhang2025NANOSLAMNG}.

Let $\mathcal{M}$ be an $n$-dimensional manifold that is locally diffeomorphic to $\mathbb{R}^n$.
To conveniently express state increments and measure discrepancies on $\mathcal{M}$, we employ
the two operators $\boxplus$ and $\boxminus$~\cite{xu2021fast}, defined as
\begin{equation}
\nonumber
\boxplus: \mathcal{M} \times \mathbb{R}^n \rightarrow \mathcal{M} ; \quad
\boxminus: \mathcal{M} \times \mathcal{M} \rightarrow \mathbb{R}^n.
\end{equation}
Taking $\mathcal{M}=\mathrm{SO}(3)$ as an example, for $\bm{R}\in\mathrm{SO}(3)$ and $\bm{r}\in\mathbb{R}^3$, the two operators are instantiated as
\begin{equation}
\nonumber
\begin{aligned}
\bm{R} \boxplus \bm{r} = \bm{R}\,\operatorname{Exp}(\bm{r}), \quad \bm{R}_1 \boxminus \bm{R}_2 = \operatorname{Log}\!\left(\bm{R}_2^{\mathsf{T}}\bm{R}_1\right).
\end{aligned}
\end{equation}
where $\operatorname{Exp}(\bm{r}) = \bm{I}
+ \frac{\sin(\|\bm{r}\|)}{\|\bm{r}\|}\,[\bm{r}]_\times
+ \frac{1-\cos(\|\bm{r}\|)}{\|\bm{r}\|^2}\,[\bm{r}]_\times^2$
is the exponential map~\cite{barfoot2024state}, and $\operatorname{Log}(\cdot)$ denotes its inverse map.

Consider a dynamical system whose state $\bm{x}\in\mathcal{M}$ evolves on a manifold:
\begin{equation}
\label{eq.model_manifold}
\begin{aligned}
\bm{x}_{t+1} &= \bm{x}_t \boxplus \big(\bm{f}(\bm{x}_t,\bm{u}_t,\bm{\xi}_t)\Delta t\big),\\
\bm{y}_t &= \bm{h}(\bm{x}_t) + \bm{\zeta}_t .
\end{aligned}
\end{equation}
We define the error state with respect to a noise-free nominal state $\bar{\bm{x}}\in\mathcal{M}$ as
\begin{equation}
\label{eq.error_state}
\delta \bm{x} \triangleq \bm{x} \boxminus \bar{\bm{x}},
\end{equation}
where $\delta\bm{x}\in\mathbb{R}^n$ lies in a Euclidean space. The error state admits its own transition model, which is typically well approximated by a linear form:
\begin{equation}
\label{eq.error_dynamics}
\delta \bm{x}_{t+1} = \bm{F}_t\,\delta \bm{x}_t + \bm{B}_t\,\bm{\xi}_t .
\end{equation}
Meanwhile, the nominal state evolves deterministically under a noise-free transition model:
\begin{equation}
\label{eq.nominal_dynamics}
\bar{\bm{x}}_{t+1} = \bar{\bm{x}}_{t} \boxplus \big(\bm{f}(\bar{\bm{x}}_{t},\bm{u}_{t},\bm{0})\Delta t\big).
\end{equation}
Consequently, manifold state estimation can be reduced to estimating the Euclidean error state $\delta\bm{x}_t$ and then composing it back onto the manifold.

\paragraph{Prediction.}
At the end of each time step, the estimator \emph{retracts} the updated error into the nominal state and resets the error to zero, i.e., $\delta \hat{\bm{x}}_{t|t}\leftarrow \bm{0}$. At the beginning of time step $t$, the nominal state $\bar{\bm{x}}_t$ is first obtained from \eqref{eq.nominal_dynamics}. The error-state prediction then follows the standard Kalman filter recursion:
\begin{equation}
\begin{aligned}
\delta \hat{\bm{x}}_{t|t-1} &= \bm{F}_{t-1}\,\delta \hat{\bm{x}}_{t-1|t-1} = \bm{0}, \\
\bm{P}_{t|t-1} &= \bm{F}_{t-1}\bm{P}_{t-1|t-1}\bm{F}_{t-1}^\top + \bm{B}_{t-1}\bm{Q}_{t-1}\bm{B}_{t-1}^\top .
\end{aligned}
\end{equation}

\paragraph{Update via NANO iterations.}
Given $\bar{\bm{x}}_t$ and the predicted covariance $\bm{P}_{t|t-1}$, we perform natural-gradient iterations on the error state. Specifically, for iteration index $i$, we update the information matrix and mean as
\begin{equation}
\label{eq.iter}
\begin{aligned}
(\bm{P}_t^{-1})^{(i+1)} &= \bm{P}_{t|t-1}^{-1} +
\Exp_{\mathcal{N}^{\delta}}
\!\left[\nabla_{\delta \bm{x}_t}^2\ell(\bar{\bm{x}}_t \boxplus \delta \bm{x}_t,\bm{y}_t)\right],  \\
\delta \hat{\bm{x}}_{t|t}^{(i+1)} &= \delta \hat{\bm{x}}_{t|t}^{(i)} -
\bm{P}_{t|t}^{(i+1)}
\Exp_{\mathcal{N}^{\delta}}
\!\left[\nabla_{\delta \bm{x}_t}\ell(\bar{\bm{x}}_t \boxplus \delta \bm{x}_t,\bm{y}_t)\right] \\
&-\bm{P}_{t|t}^{(i+1)}\bm{P}_{t|t-1}^{-1}\delta \hat{\bm{x}}_{t|t}^{(i)},
\end{aligned}
\end{equation}
where $\mathcal{N}^{\delta}$ is shorthand for $\mathcal{N}(\delta \bm{x}_t;\,\delta \hat{\bm{x}}_{t|t}^{(i)},\,\bm{P}_{t|t}^{(i)})$.

\paragraph{Composition and reset.}
After convergence (or a fixed number of iterations), the final manifold estimate is obtained by composing the nominal state with the refined error estimate:
\begin{equation}
\label{eq.final_com}
\hat{\bm{x}}_{t|t} = \bar{\bm{x}}_t \boxplus \delta \hat{\bm{x}}_{t|t}^{(i+1)} .
\end{equation}
Finally, we reset the error state $\delta \bm{x}_t \leftarrow \bm{0}$ and proceed to the next time step.

\section{Illustrative Examples}\label{sec.example}

In this section, we present four representative real-world applications that highlight the practical advantages of the NANO framework. The first case study considers satellite attitude estimation. The second presents NANO-SLAM, a variant tailored for direct deployment in vehicle SLAM. The third introduces NANO-L, a Lie-group–aware extension designed to operate on Lie group manifolds for the quadruped robot state estimation. The fourth demonstrates the application of NANO to humanoid robot state estimation. Together, these four case studies provide a comprehensive evaluation of NANO in terms of accuracy, robustness, and computational efficiency.
\subsection{Satellite Attitude Estimation}
\subsubsection{System Modeling}
The satellite's attitude is described by Euler angles \cite{brossard2020ukfm}. The state variables are defined as \(\theta = [p, r, y]\), where \(p\), \(r\), and \(y\) correspond to pitch, yaw, and roll angles, respectively. The system input is the satellite's angular velocity, denoted as \(\omega\). The satellite is equipped with a gravimeter and a magnetometer, which measure the components of the gravitational and geomagnetic fields at the satellite's attitude. The system's state transition equation and observation equation are as follows:

\[
\theta_{t+1} = \theta_t + M(\theta_t) \omega_t \Delta t + \xi_t
\]
\[
y_t =
\begin{bmatrix}
R(\theta_t)^{\top} g \\
R(\theta_t)^{\top} b
\end{bmatrix}
+ \zeta_t
\]
where \(\Delta t = 0.01s\), \(g = [0, 0, -9.81]^T\), and \(b = [27.75, -3.65, 47.21]^T\) are the Earth's gravitational and geomagnetic intensities. \(M(\theta_t)\) and \(R(\theta_t)\) are the transition matrices between the Earth coordinate system and the satellite body coordinate system, defined as:

\[
M(\theta_t) =
\begin{bmatrix}
1 & \frac{\sin(p) \sin(r)}{\cos(p)} & \frac{\cos(r) \sin(p)}{\cos(p)} \\
0 & \cos(r) & -\sin(r) \\
0 & \frac{\sin(r)}{\cos(p)} & \frac{\cos(r)}{\cos(p)}
\end{bmatrix}
\]
\[
R(\theta_t) = R_p R_r R_y
\]
where \(R_p\), \(R_r\), and \(R_y\) are the rotation matrices defined as:

\[
R_p = \begin{bmatrix} 
\cos(p) & 0 & \sin(p) \\
0 & 1 & 0 \\
-\sin(p) & 0 & \cos(p) 
\end{bmatrix}
\]
\[
R_r = \begin{bmatrix} 
1 & 0 & 0 \\
0 & \cos(r) & -\sin(r) \\
0 & \sin(r) & \cos(r) 
\end{bmatrix}
\]
\[
R_y = \begin{bmatrix} 
\cos(y) & -\sin(y) & 0 \\
\sin(y) & \cos(y) & 0 \\
0 & 0 & 1 
\end{bmatrix}
\]

In this experiment, we assume that the process noise follows a Laplace distribution and the observation noise follows a Gaussian distribution. Specifically, the process noise has a 10\% chance of being contaminated, with \(\xi_t \sim 0.9 \cdot \text{Laplace}(0, 10^{-5} I_{3 \times 3}) + 0.1 \cdot \text{Laplace}(0, 10^{-2} I_{3 \times 3})\), while the observation noise has a 15\% chance of being contaminated with a beta distribution, with \(\zeta_t \sim 0.85 \cdot N(0, 10^{-4} I_{3 \times 3}) + 0.15 \cdot \text{beta}(1.2, 1.5)\).
Such designs are used to test the performance of NANO under non-Gaussian noises.

\subsubsection{Experimental Setup}
We consider two angular velocity inputs:  
\(\omega_t = \frac{\pi}{18} \sin(2 \Delta t \pi t) \cdot I_{3 \times 1}\), and \(\omega_t = \frac{\pi}{18\sqrt2} \cdot I_{3 \times 1}\).  
We also consider two initial state estimation scenarios: one with an accurate initial state and the other with a biased initial state. The initial state is set as \(x_0 \sim N(0, 10^{-3} I_{3 \times 3})\), while the initial state estimate \(\hat{x}_0\) is set to \([0, 0, 0]\) and \([10^\circ, 10^\circ, 10^\circ]\), respectively. We compare the performance of the NANO, UKF, and EKF filters under these scenarios. The accuracy of the estimation is evaluated using the RMSE metric.

\subsubsection{Comparison Results}
The experimental results corresponding to the two angular velocity input conditions described in Section 2 are summarized in Table IV and Table V, respectively. It can be observed that NANO consistently achieves lower RMSE than EKF and UKF, demonstrating superior estimation accuracy under both input conditions.
\begin{table}[ht]
\fontsize{8.2}{8.2}\selectfont
\caption{Comparison of RMSE and Computation Time under Sinusoidal input \(\omega_t = \frac{\pi}{18} \sin(2 \Delta t \pi t) \cdot I_{3 \times 1}\)}
\label{table:ate_re}
\centering
\begin{threeparttable}[t]
\renewcommand{\arraystretch}{1.3} 
\begin{tabular}{ c|c| c c c  c c c | c } 
\toprule[1pt]
$ \hat{x}_0 $ & Method & RMSE & Time [ms] \\ \midrule[0.5pt]
\multirow{3}{*}{$ [0 \, 0 \, 0]^T $} & EKF & 0.944 (0.020) & 0.251 \\ \cline{2-4} 
 & UKF & 0.944 (0.020) & 0.387 \\ \cline{2-4} 
 & NANO & 0.589 (0.006) & 3.084 \\ \midrule[0.5pt]
\multirow{3}{*}{$ [\frac{\pi}{18} \, \frac{\pi}{18} \, \frac{\pi}{18}]^T $} & EKF & 1.576 (0.017) & 0.268 \\ \cline{2-4}
 & UKF & 1.576 (0.017) & 0.429 \\ \cline{2-4}
 & NANO & 1.390 (0.018) & 3.189 \\ 
\bottomrule[1pt]
\end{tabular}
\end{threeparttable}
\end{table}

\begin{table}[ht]
\fontsize{8.2}{8.2}\selectfont
\caption{Comparison of RMSE and Computation Time under constant input \(\omega_t = \frac{\pi}{18\sqrt2} \cdot I_{3 \times 1}\)}
\label{table:ate_re}
\centering
\begin{threeparttable}[t]
\renewcommand{\arraystretch}{1.3} 
\begin{tabular}{ c|c| c c c  c c c | c } 
\toprule[1pt]
$ \hat{x}_0 $ & Method & RMSE & Time [ms] \\ \midrule[0.5pt]
\multirow{3}{*}{$ [0 \, 0 \, 0]^T $} & EKF & 0.971 (0.021) & 0.262 \\ \cline{2-4} 
 & UKF & 0.971 (0.021) & 0.415 \\ \cline{2-4} 
 & NANO & 0.597 (0.007) & 2.946 \\ \midrule[0.5pt]
\multirow{3}{*}{$ [\frac{\pi}{18} \, \frac{\pi}{18} \, \frac{\pi}{18}]^T $} & EKF & 1.592 (0.018) & 0.248 \\ \cline{2-4}
 & UKF & 1.592 (0.018) & 0.383 \\ \cline{2-4}
 & NANO & 1.391 (0.018) & 3.093 \\ 
\bottomrule[1pt]
\end{tabular}
\end{threeparttable}
\end{table}

\begin{figure}[ht]
\centering

\begin{subfigure}{0.48\linewidth}
    \centering
    \includegraphics[width=\linewidth]{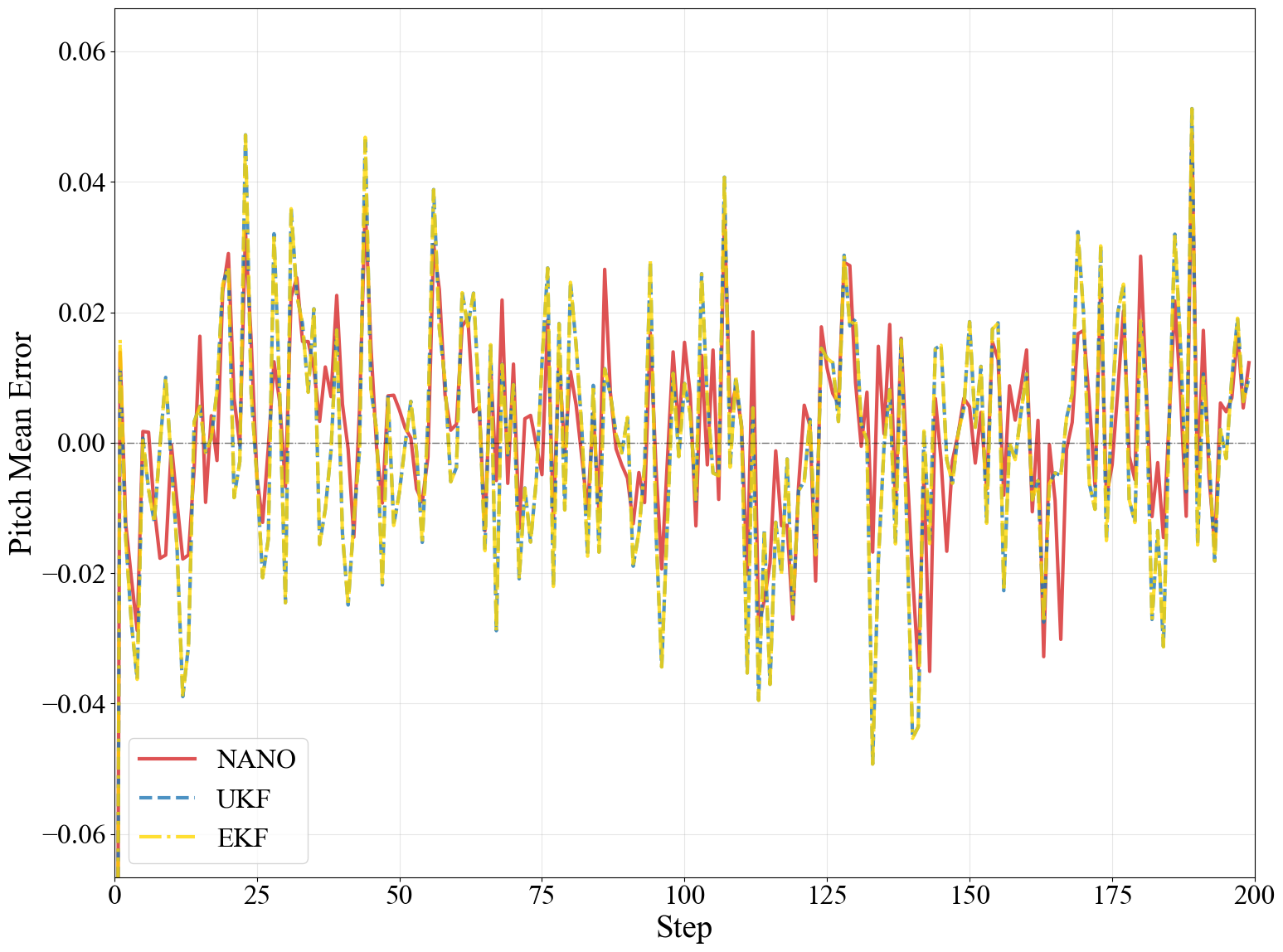}
    \caption{Mean pitch error at each time step}
\end{subfigure}
\hfill
\begin{subfigure}{0.48\linewidth}
    \centering
    \includegraphics[width=\linewidth]{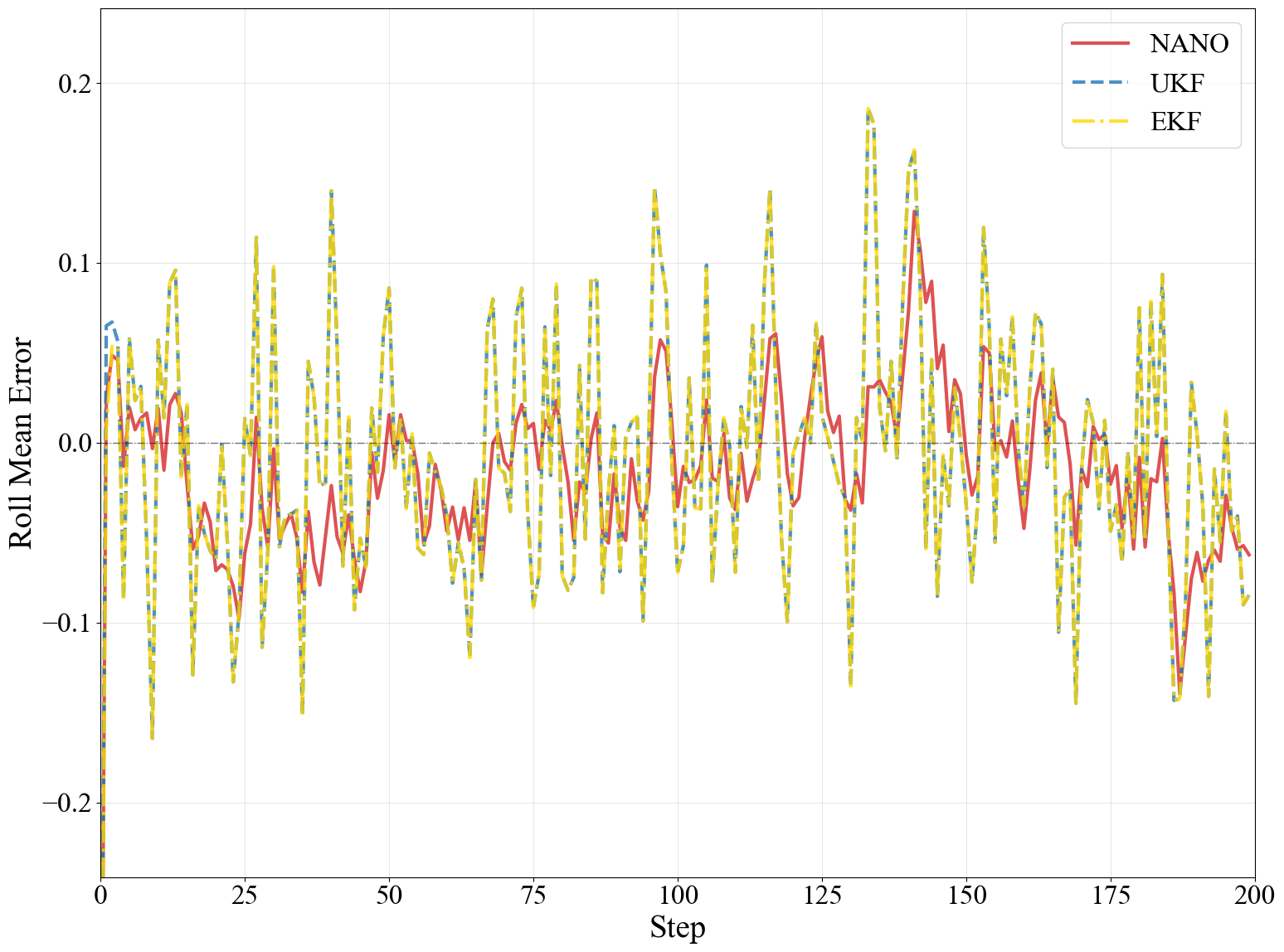}
    \caption{Mean roll error at each time step}
\end{subfigure}

\vspace{0.3cm}

\begin{subfigure}{0.48\linewidth}
    \centering
    \includegraphics[width=\linewidth]{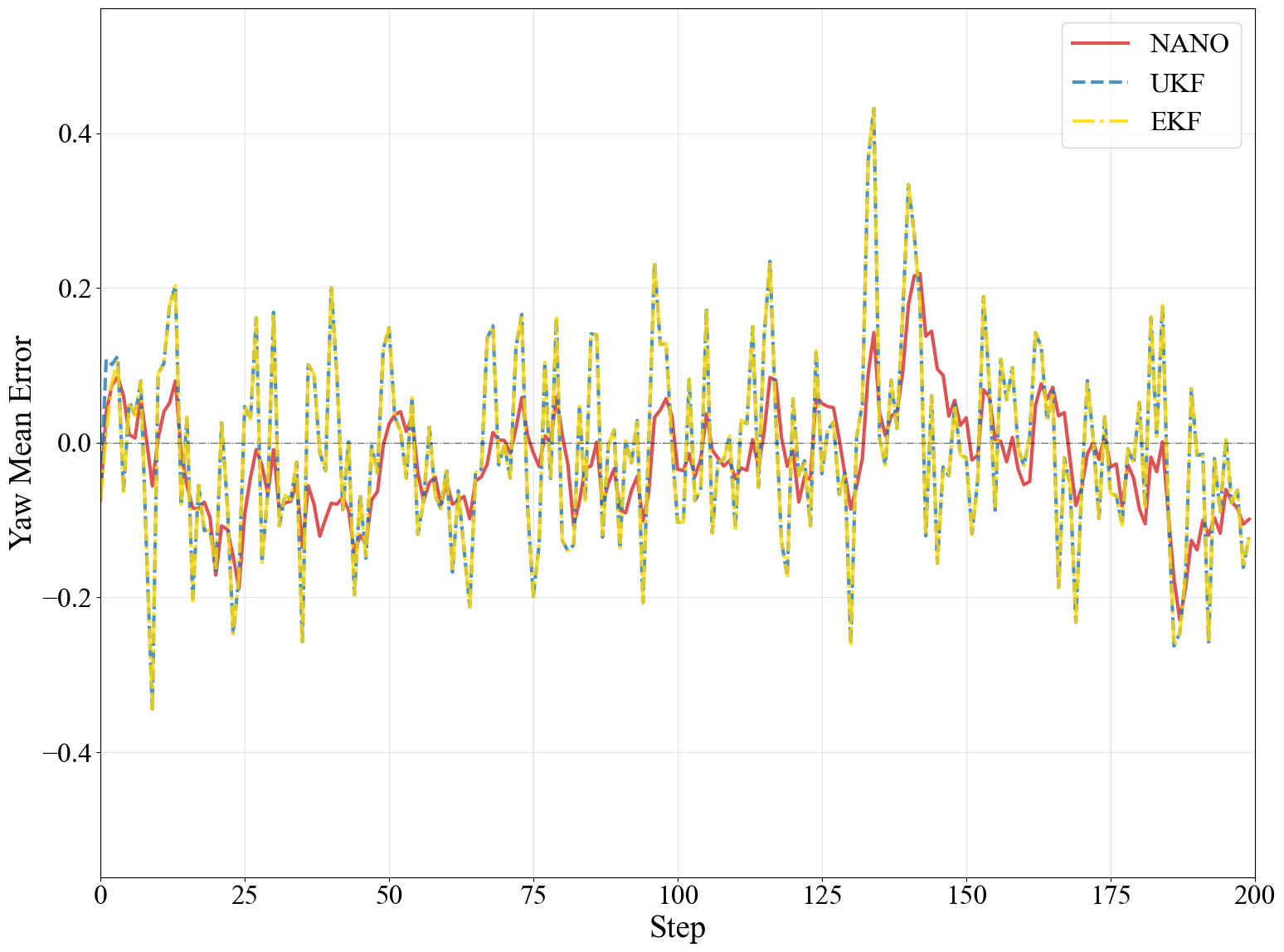}
    \caption{Mean yaw error at each time step}
\end{subfigure}
\hfill
\begin{subfigure}{0.48\linewidth}
    \centering
    \includegraphics[width=\linewidth]{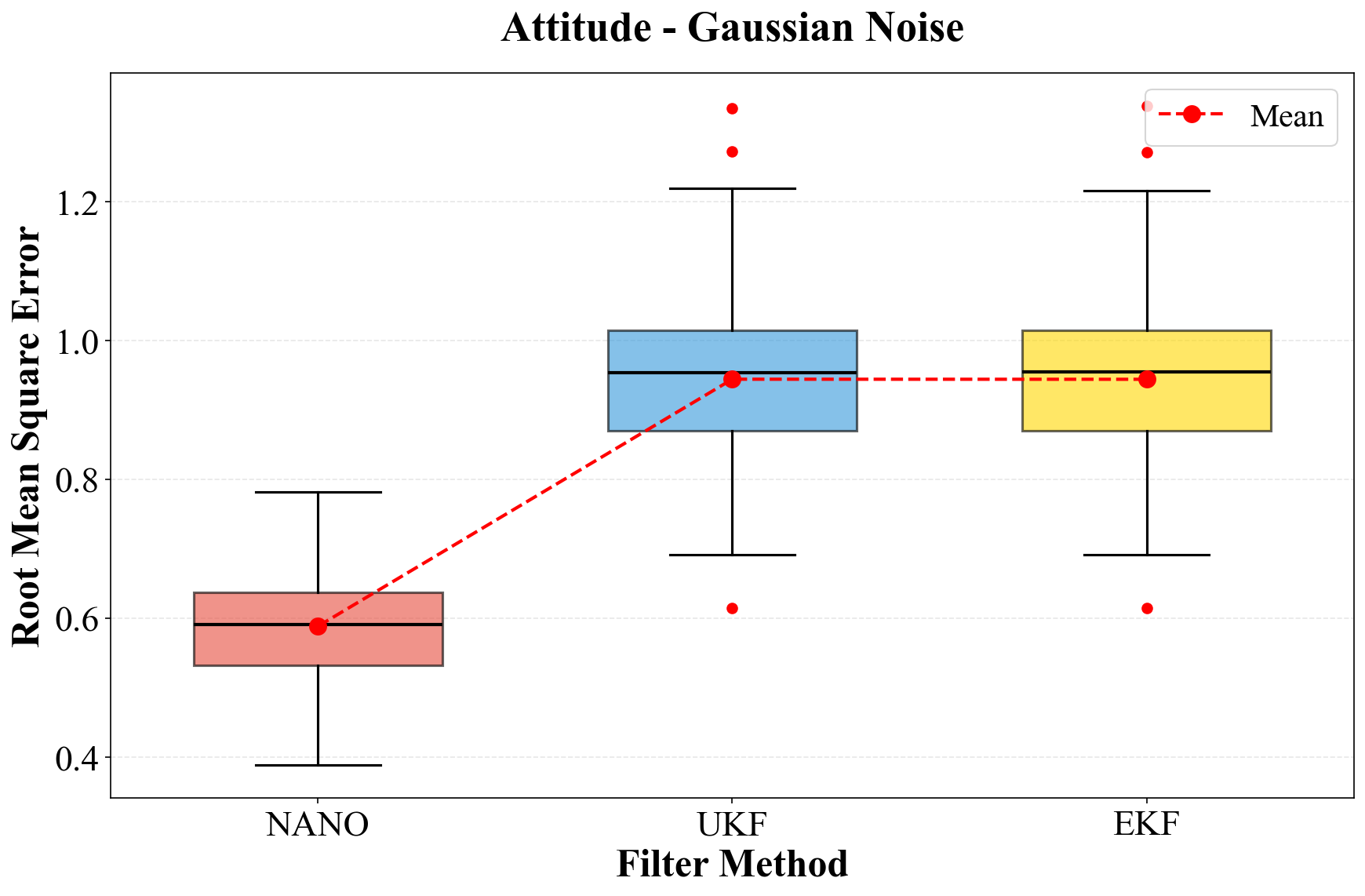}
    \caption{RMSE boxplot for each experiment}
\end{subfigure}

\caption{Comparison of attitude estimation performance under Sinusoidal input and accurate initialization.}
\label{fig:comparison}
\end{figure}

When the initial state estimate is accurate ($\hat{x}_0 = [0,0,0]^T$), EKF and UKF exhibit nearly identical RMSE values, whereas NANO achieves significantly lower errors. This result indicates that NANO is more effective in suppressing the influence of non-Gaussian process noise and contaminated observation noise. This advantage is further illustrated in Fig.~\ref{fig:comparison}(a)-(c), where the time-series error curves under sinusoidal input show that NANO maintains smaller fluctuations in pitch, roll, and yaw throughout the entire time horizon. Moreover, the RMSE boxplot in Fig.~\ref{fig:comparison}(d) demonstrates that NANO has a lower median and a more compact interquartile range, confirming its superior accuracy and estimation stability from a statistical perspective.

When the initial estimate is biased ($\hat{x}_0 = [\frac{\pi}{18}, \frac{\pi}{18}, \frac{\pi}{18}]^T$), the estimation errors of all filters increase. Nevertheless, NANO still maintains the lowest RMSE among the three methods, highlighting its superior robustness to initialization errors and stronger convergence capability.

More importantly, under both sinusoidal and constant angular velocity inputs, NANO consistently preserves its performance advantage. This demonstrates that the improvement achieved by NANO is not dependent on a specific motion pattern, but rather reflects a fundamentally stronger estimation capability across different dynamic conditions, validating its robustness and generalization ability.

\subsection{NANO for Vehicle SLAM}
\subsubsection{System Modeling}
The autonomous vehicle in a 2-dimensional (2D) coordinate system is illustrated in Fig.~\ref{fig.vehicle}. The vehicle position is denoted by $(p_x, p_y)$ and its heading is $\theta$. The forward speed is $v$ and the steering input is $\alpha$. Moreover, $L$ denotes the wheelbase and $H$ the track width. The discrete-time vehicle dynamics follow the standard Ackermann kinematic model:
\begin{equation}
\label{eq.motion_model}
\begin{aligned}
x_{t+1}
&= x_t + \begin{bmatrix}
\Delta T \left( v_t \cos\theta_t - \frac{v_t}{L} \tan\theta_t (\alpha_t \sin\theta_t + b \cos\theta_t) \right) \\
\Delta T \left( v_t \sin\theta_t + \frac{v_t}{L} \tan\theta_t (\alpha_t \cos\theta_t + b \sin\theta_t) \right) \\
 \Delta T \frac{v_t}{L} \tan(\alpha_t) 
\end{bmatrix}
 + \xi_t.
\end{aligned}
\end{equation}
Here, the vehicle state is $x = [p_x, p_y, \theta]^{\top} \in \mathbb{R}^3$, and the control input is $u = [v, \alpha]^{\top}\in \mathbb{R}^2$. The process disturbance is modeled as zero-mean Gaussian noise, $\xi_t \sim \mathcal{N}(0, Q_t)$ with $Q_t \in \mathbb{R}^{3\times3}$, and $\Delta T$ is the sampling interval. Note that $v$ corresponds to the velocity at the axle center and is not directly observed; instead, it is computed from the wheel encoder speed $v_e$. In practice, we use
\begin{equation}
v = \frac{v_e}{ 1 -  H \cdot \tan{\alpha}/ (2L)}.
\end{equation}
We consider range-bearing observations between the vehicle and landmarks, obtained from an onboard laser range finder. At each time step, the sensor outputs $K$ measurements \(y_t = \{y_{k,t}\}_{k=1}^K\) with \(y_{k,t}\in \mathbb{R}^2\). The map contains $M$ landmarks \(m_{1:M} =\{m_{j}\}_{j=1}^M \), where each landmark is $m_j = \left[m_{x,j}, \, m_{y,j} \right]^\top \in \mathbb{R}^2$. 
If the $k$-th observation at time $t$ is associated with landmark $j$, the measurement model is
\begin{equation}
\label{eq.meas_model}
y_{k,t} = g(x_t, m_{j}) + \zeta_{k,t},     
\end{equation}
with
\begin{equation}
\begin{aligned}
\label{eq.meas_func}
g(x_t, m_{j}) =   \begin{bmatrix}
\sqrt{(m_{x,j} - p_{x,t})^2 + (m_{y,j} - p_{y,t})^2}\\
\arctan{\frac{m_{y,j} - p_{y,t}}{m_{x,j} - p_{x,t}}} - \theta_t
\end{bmatrix}
\end{aligned}.
\end{equation}
The measurement noise is assumed Gaussian, $\zeta_{k,t}\sim\mathcal{N}(0,R_t)$ with covariance $R_t\in\mathbb{R}^{2\times2}$.

\subsubsection{Experimental Results}
We evaluate NANO-SLAM on the Sydney Victoria Park dataset to examine its behavior in a realistic, large-scale environment. The platform is equipped with wheel encoders, a laser range finder, and GPS, and it traverses the park along a trajectory longer than 3.5\,km (Fig.~\ref{fig.map}). Encoder readings provide velocity and steering-related inputs, while the laser range finder scans a 180$^\circ$ frontal field of view and produces range--bearing measurements for landmark detection; in this dataset, nearby trees serve as natural landmarks. GPS measurements are treated as reference to compute localization error. The vehicle wheelbase is 2.83\,m and the track width is 0.76\,m.

\begin{figure}[t]
\centering
\includegraphics[width=0.66\linewidth]{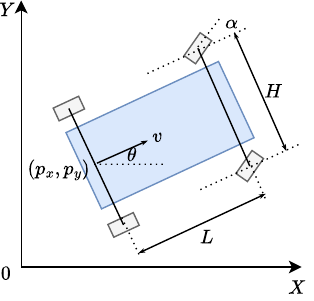}
\caption{Vehicle in the 2D coordinate system.}
\label{fig.vehicle}
\end{figure}

\begin{figure}[t]
\centering
\includegraphics[width=0.7\linewidth]{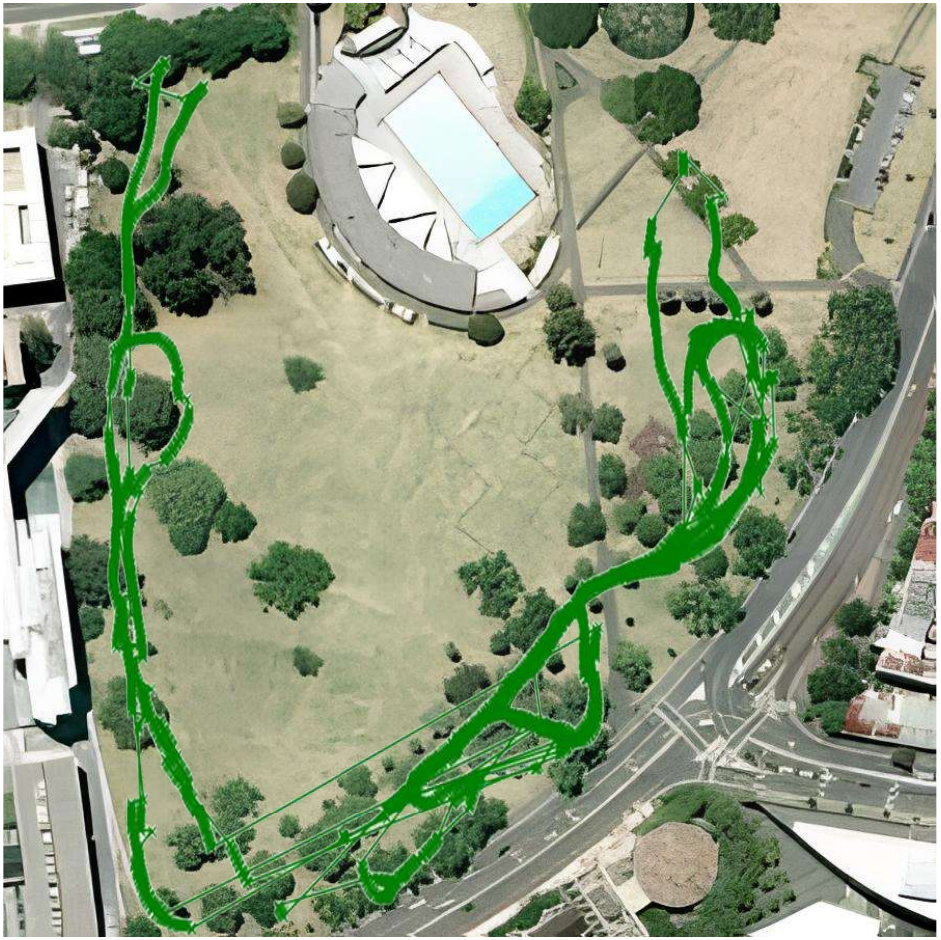}
\caption{Victoria Park map and vehicle trajectory obtained from GPS data.}
\label{fig.map}
\end{figure}

\begin{figure}[htbp]
\centering
\subfloat[]{
\includegraphics[width=0.45\textwidth]{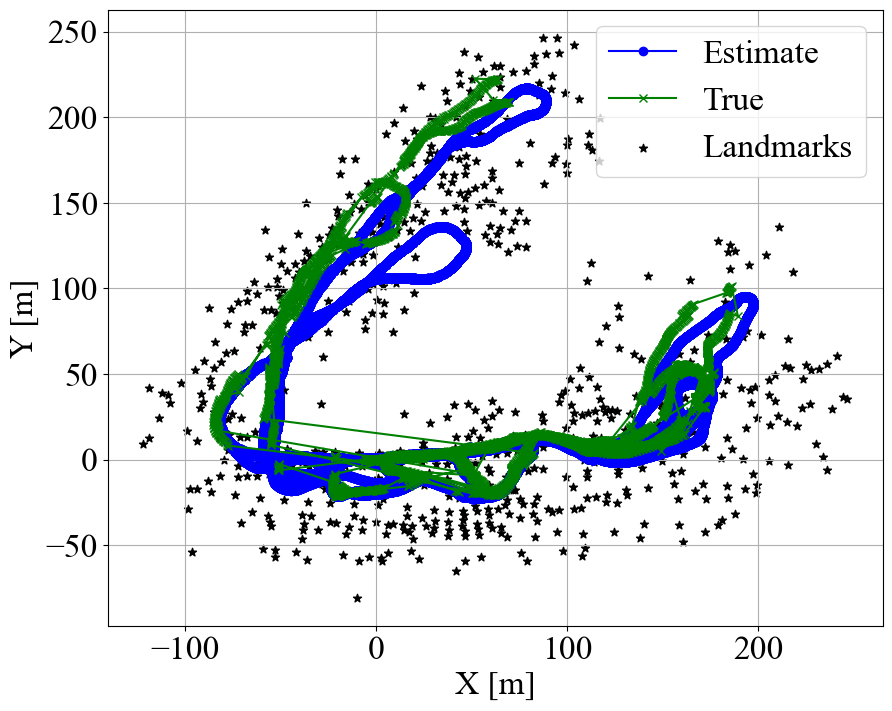}
\label{est:1}
}
\\
\subfloat[]{
\includegraphics[width=0.45\textwidth]{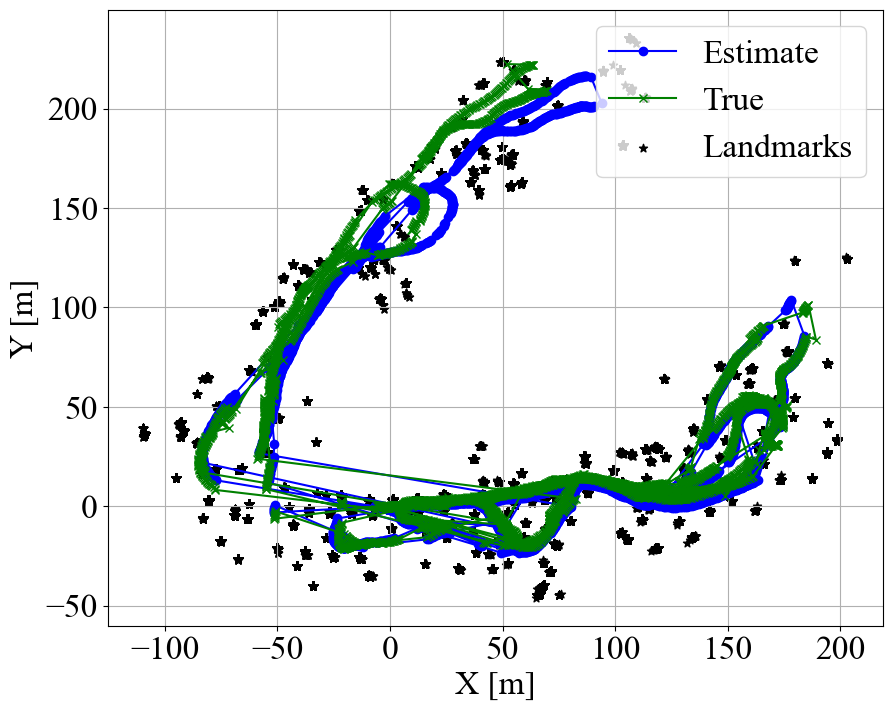}
\label{est:2}
}
\\
\subfloat[]{
\includegraphics[width=0.45\textwidth]{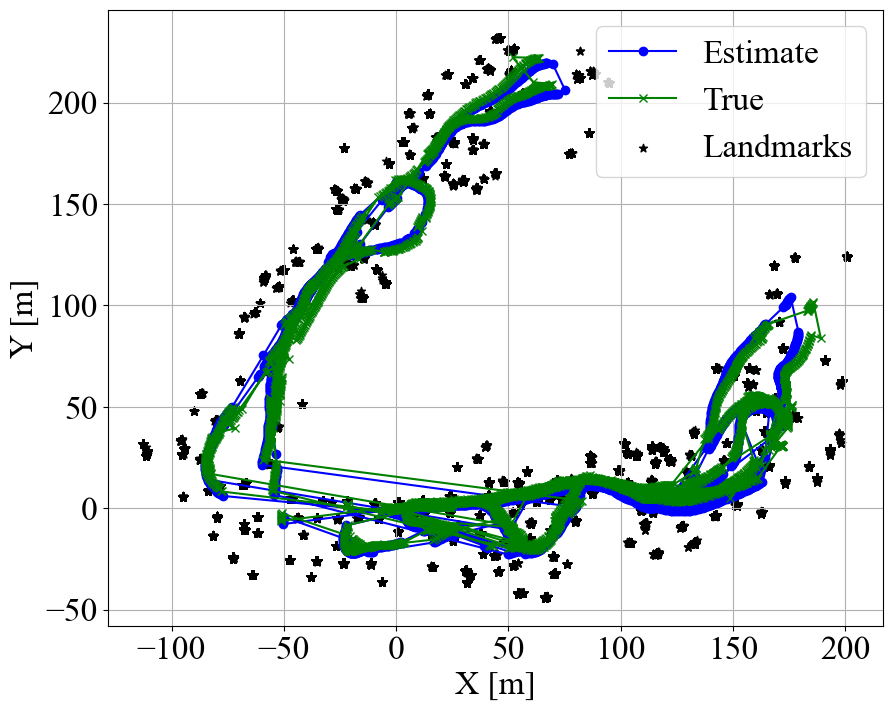}
\label{est:3}
}
\caption{Experimental results of the vehicle localization and landmark location estimates. (a) EKF-SLAM. (b) UFastSLAM. (c) NANO-SLAM}
\label{fig.estimate}
\end{figure}

We compare NANO-SLAM with EKF-SLAM~\cite{huang2007convergence} and UFastSLAM~\cite{kim2008unscented}, which represent a conventional linearization-based SLAM approach and an RBPF-based alternative, respectively. To account for practical disturbances (e.g., wheel slip and uneven terrain), we set the standard deviations of velocity and steering noises to $\sigma_v = 2$\,m/s and $\sigma_g = 6^\circ$. The range and bearing channels use $\sigma_r = 1$\,m and $\sigma_b = 3^\circ$ to model sensing and environmental uncertainty. Both UFastSLAM and NANO-SLAM employ 10 particles. Localization accuracy is quantified using the RMSE:
\begin{equation}
\label{eq.rmse}
\mathrm{RMSE} = \sqrt{\frac{1}{T}\sum_{t=1}^T\|p_t - \hat{p}_t\|_2^2},
\end{equation}
where $p_t$ and $\hat{p}_t$ denote the ground-truth and estimated positions at time step $t$, and $T$ is the total number of time steps.

\begin{table}[htbp]
\centering
\caption{Comparison of different vehicle SLAM methods.}
\label{tab:slam_comparison}
\begin{tabular}{lcc}
\toprule
\textbf{Method} & \textbf{RMSE [m]} & \textbf{Time [ms]} \\
\midrule
EKF-SLAM    & 7.783 & 312.79 \\
\midrule
UFastSLAM   & 5.147 & \textbf{14.896} \\
\midrule
NANO-SLAM   & \textbf{2.538} & 17.692 \\
\bottomrule
\end{tabular}
\end{table}

As reported in Table~\ref{tab:slam_comparison}, NANO-SLAM achieves the best localization accuracy, reducing RMSE to 2.538\,m, which is more than a 50\% improvement over UFastSLAM (5.147\,m), the next strongest baseline. This improvement is mainly attributed to avoiding measurement-model linearization and the resulting approximation error. The runtime remains close to UFastSLAM (17.692\,ms vs. 14.896\,ms), corresponding to an overhead of about 18\%. Fig.~\ref{fig.estimate} further visualizes the estimated trajectories and landmarks, where NANO-SLAM produces the trajectory that most closely aligns with the GPS reference.

\subsection{Quadruped Robot State Estimation}
We consider a quadruped robot state estimation problem whose state naturally evolves on a Lie group manifold \cite{Zhang2026NaturalGL}. This setting is practically challenging due to discontinuous foot--ground contacts and occasional outliers caused by imperfect contact detection \cite{bloesch2013state,hartley2020contact}. 

\subsubsection{System Modeling}
The robot base orientation, velocity, and position in the world frame are denoted by $\bm{R}_t\in \mathrm{SO}(3)$, $\bm{v}_t\in\mathbb{R}^3$, and $\bm{p}_t\in\mathbb{R}^3$. To incorporate leg kinematics, we further include the world-frame position of a generic foot contact point $\bm{s}_t\in\mathbb{R}^3$. These variables can be lifted to a compact matrix form on $\mathrm{SE}_3(3)$:
\begin{equation}
\label{eq:leg_state}
\bm{x}_t \;=\;
\begin{bmatrix}
\bm{R}_t & \bm{v}_t & \bm{p}_t & \bm{s}_t \\
\bm{0}_{3\times3} & \bm{0}_{3\times1} & \bm{0}_{3\times1} & \bm{0}_{3\times1} 
\end{bmatrix}
\in \mathrm{SE}_3(3).
\end{equation}

The onboard IMU provides body-frame angular velocity $\tilde{\bm{\omega}}_t$ and linear acceleration $\tilde{\bm{a}}_t$ \cite{barfoot2024state,hartley2020contact}. The continuous-time dynamics admit the following Lie-group form:
\begin{equation}
\label{eq:leg_process}
\begin{aligned}
\frac{\mathrm{d}}{\mathrm{d}t}\bm{x}_t 
&= \bm{f}_{\bm{u}_t}(\bm{x}_t) + \bm{x}_t\bm{\xi}_t^\wedge \\
&= \begin{bmatrix}
\bm{R}_t \tilde{\bm{\omega}}_t^\wedge  & \bm{R}_t \tilde{\bm{a}}_t + \bm{g} & \bm{v}_t & \bm{0}_{3 \times 1} \\
\bm{0}_{3\times3} & \bm{0}_{3\times1} & \bm{0}_{3\times1} & \bm{0}_{3\times1} 
\end{bmatrix}
+ \bm{x}_t\bm{\xi}_t^\wedge ,
\end{aligned}
\end{equation}
where $\bm{g}\in\mathbb{R}^3$ is gravity. The process noise $\bm{n}_t$ is modeled as zero-mean Gaussian with block-diagonal covariance, capturing IMU noise and contact-point perturbations \cite{bloesch2013state,hartley2020contact}. In practice, \eqref{eq:leg_process} is discretized via standard Lie-group integration over the sampling interval.

Joint encoders provide leg joint angles, which are mapped through forward kinematics to obtain the foot position relative to the body. The same relative position can be computed from the state as $\bm{R}_t^\top(\bm{s}_t-\bm{p}_t)$. This yields the leg-odometry measurement model
\begin{equation}
\label{eq:leg_meas_simple}
\bm{y}_t \;=\; \bm{R}_t^\top(\bm{s}_t-\bm{p}_t) \;+\; \bm{\zeta}_t,
\qquad
\bm{\zeta}_t \sim \mathcal{N}(\bm{0},\bm{\Gamma}_t),
\end{equation}
where $\bm{\Gamma}_t$ captures encoder and kinematic uncertainty. The above model can be written in an invariant form
\begin{equation}
\label{eq:leg_meas_invariant}
\bm{y}_t \;=\; \bm{g}(\bm{x}_t) + \bm{\zeta}_t,
\qquad
\bm{g}(\bm{x}_t) := \bm{x}_t^{-1}\bm{b},
\end{equation}
where $\bm{b}$ is a constant vector selecting the corresponding relative-position component \cite{barrau2016invariant,hartley2020contact}. This invariant structure also applies to other relative-position sensing modalities such as visual or LiDAR odometry \cite{xu2021fast}.

\subsubsection{Experimental Setup}
As shown in Fig.~\ref{fig:env}, two real-world test environments are used: a \textit{flat} terrain as the nominal setting, and an \textit{unstable} surface that induces uncertain contact conditions and challenges estimator robustness.

In both environments, the robot executes a trot gait while recording multimodal onboard measurements at 200\,Hz, including one IMU, 12 joint encoders, and four foot contact sensors. Ground-truth pose is provided by a NOKOV motion-capture system at 100\,Hz. For each environment, five sequences of 60\,s are collected to reduce the impact of run-to-run variations. All methods are evaluated offline on the same datasets using a laptop with an Intel Core i9-14900HX CPU, and are initialized from the ground-truth state. Noise settings and algorithm hyperparameters are kept fixed across all runs, as summarized in Table~\ref{table.param}.

The comparison focuses on EKF~\cite{bloesch2013state} and InEKF~\cite{hartley2020contact}, two widely adopted baselines for legged-robot state estimation. UKF-style methods are not included since they are seldom deployed on such platforms due to higher computational overhead and reduced numerical robustness. Performance is quantified using absolute trajectory error (ATE) and relative error (RE) for position, velocity, and orientation, following the definitions in \cite{zhang2018tutorial}. ATE captures global trajectory consistency with respect to the ground truth, while RE measures local drift over a fixed interval of 3\,s.

\begin{figure}[!t]
\centering
\includegraphics[width=0.5\textwidth]{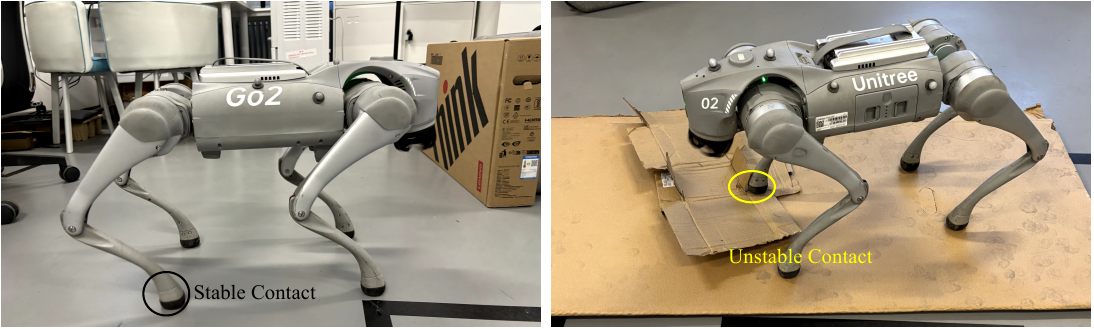}
\caption{Different environments of real-world legged robot experiments.}
\label{fig:env}
\end{figure}

\begin{table}[h]
\fontsize{9}{9}\selectfont
\caption{Noise parameters and NANO-L's hyperparameter.}
\label{table.param}
\centering
\begin{threeparttable}[t]
\renewcommand{\arraystretch}{1.2}
\begin{tabular}{ c c c} 
\toprule[1pt]
Name & Symbol & Value \\
\midrule[0.5pt]
Accelerometer & $\sigma_a$ & $0.2568\ \rm{m/s^2}$  \\
\midrule[0.5pt]
Gyroscope & $\sigma_\omega$ & $0.00139\ \mathrm{rad/s}$  \\
\midrule[0.5pt]
Joint encoder & $\sigma_{e}$ & $0.3\ \rm{rad}$  \\
\midrule[0.5pt]
Slip & $\sigma_{s}$ & $0.001\ \rm{m/s}$  \\
\midrule[0.5pt]
Stopping threshold & $\gamma$ & $10^{-4}$  \\
\bottomrule[1pt]
\end{tabular}
\end{threeparttable}
\end{table}

\subsubsection{Comparison Results}
Table~\ref{table:ate_re} reports ATE and RE over both environments, showing the mean (outside parentheses) and variance (inside parentheses) computed from five trials. As expected, all filters exhibit larger errors on the unstable surface, confirming the increased difficulty under imperfect contact conditions. Across both terrains, manifold-aware estimators provide clear gains over the standard EKF, most notably in position accuracy. Moreover, NANO-L achieves the best overall performance, outperforming InEKF on nearly all metrics while maintaining the smallest degradation when moving from flat to unstable terrain. On the unstable surface, for example, NANO-L reduces the position ATE by about 41\% relative to InEKF.

Fig.~\ref{fig:unstable} visualizes trajectories from a representative unstable-terrain run (EKF omitted for clarity). The qualitative comparison is consistent with Table~\ref{table:ate_re}, where NANO-L remains closer to the motion-capture reference across state components. The vertical position $p_z$ shows drift for all methods, which aligns with the limited observability of absolute height from leg odometry \cite{yoon2023invariant}; nevertheless, NANO-L maintains a noticeable tracking advantage.

Fig.~\ref{fig:traj} further compares 2D localization on the flat terrain across runs with similar speed ($\approx 0.2$\,m/s) but different travel distances and durations. The results highlight that legged-robot odometry accumulates drift over time, whereas NANO-L mitigates error growth by avoiding observation-model linearization, leading to markedly reduced trajectory drift. In terms of runtime, InEKF is the fastest, but NANO-L processes each step in about 3.35\,ms, remaining comfortably within the 5\,ms budget required for 200\,Hz operation.

\begin{table*}[!t]
\fontsize{8.2}{8.2}\selectfont
\caption{Comparison of ATE, RE, and Computation Time Across Real-World Environments}
\label{table:ate_re}
\centering
\begin{threeparttable}[t]
\renewcommand{\arraystretch}{1.3} 
\begin{tabular}{ c|c| c c c  c c c | c } 
\toprule[1pt]
 \multirow{2} {*}{Environment} & \multirow{2} {*}{Method} & $\rm ATE_{pos}$ & $\rm ATE_{vel}$ & $\rm ATE_{ori}$ & $\rm RE_{pos}$ & $\rm RE_{vel}$ & $\rm RE_{ori}$ & Time\\
 & & $\left [ \rm m \right ]$ & $\left [ \rm m/s \right ]$ & $\left [ \rm rad \right ]$ & $\left [ \rm m \right ]$ & $\left [ \rm m/s \right ]$ & $\left [ \rm rad \right ]$ & $\left [ \rm ms \right ]$ \\
\midrule[0.5pt]
\multirow{3} {*}{Flat} 
& EKF & $0.373\ (0.055)$ & $0.122\ (0.032)$ & $0.020\ (0.005)$ & $0.076\ (0.010)$ & $0.107\ (0.008)$ & $0.017\ (0.005)$ & 3.905\\
& InEKF & $0.281\ (0.099)$ & $0.108\ (0.007)$ & $0.022\ (0.004)$ & $0.072\ (0.010)$ & $0.106\ (0.015)$ & $0.017\ (0.004)$ & \textbf{0.712}\\
& NANO-L & $\bf 0.212$ $(0.078)$ & $\bf 0.107$ $(0.009)$ & $\bf 0.017$ $(0.005)$ & $\bf 0.065$ $(0.004)$ & $\bf 0.093$ $(0.014)$ & $\bf 0.016$ $(0.003)$ & 3.351\\
\midrule[0.5pt]
\multirow{3} {*}{Unstable} 
& EKF & $0.633\ (0.200)$ & $0.149\ (0.057)$ & $0.021\ (0.002)$ & $0.229\ (0.163)$ & $0.094\ (0.014)$ & $0.016\ (0.003)$ & 3.905\\
& InEKF & $0.401\ (0.108)$ & $0.140\ (0.028)$ & $0.023\ (0.001)$ & $0.162\ (0.082)$ & $0.106\ (0.010)$ & $0.017\ (0.003)$ & \textbf{0.712}\\
& NANO-L & $\bf 0.236$ $(0.058)$ & $\bf 0.125$ $(0.011)$ & $\bf 0.012$ $(0.001)$ & $\bf 0.104$ $(0.031)$ & $\bf 0.108$ $(0.006)$ & $\bf 0.014$ $(0.004)$ & 3.351\\
\bottomrule[1pt]
\end{tabular}
\end{threeparttable}
\end{table*}

\begin{figure*}[!t]
\centering
\subfloat{
    \includegraphics[width=0.33\textwidth]{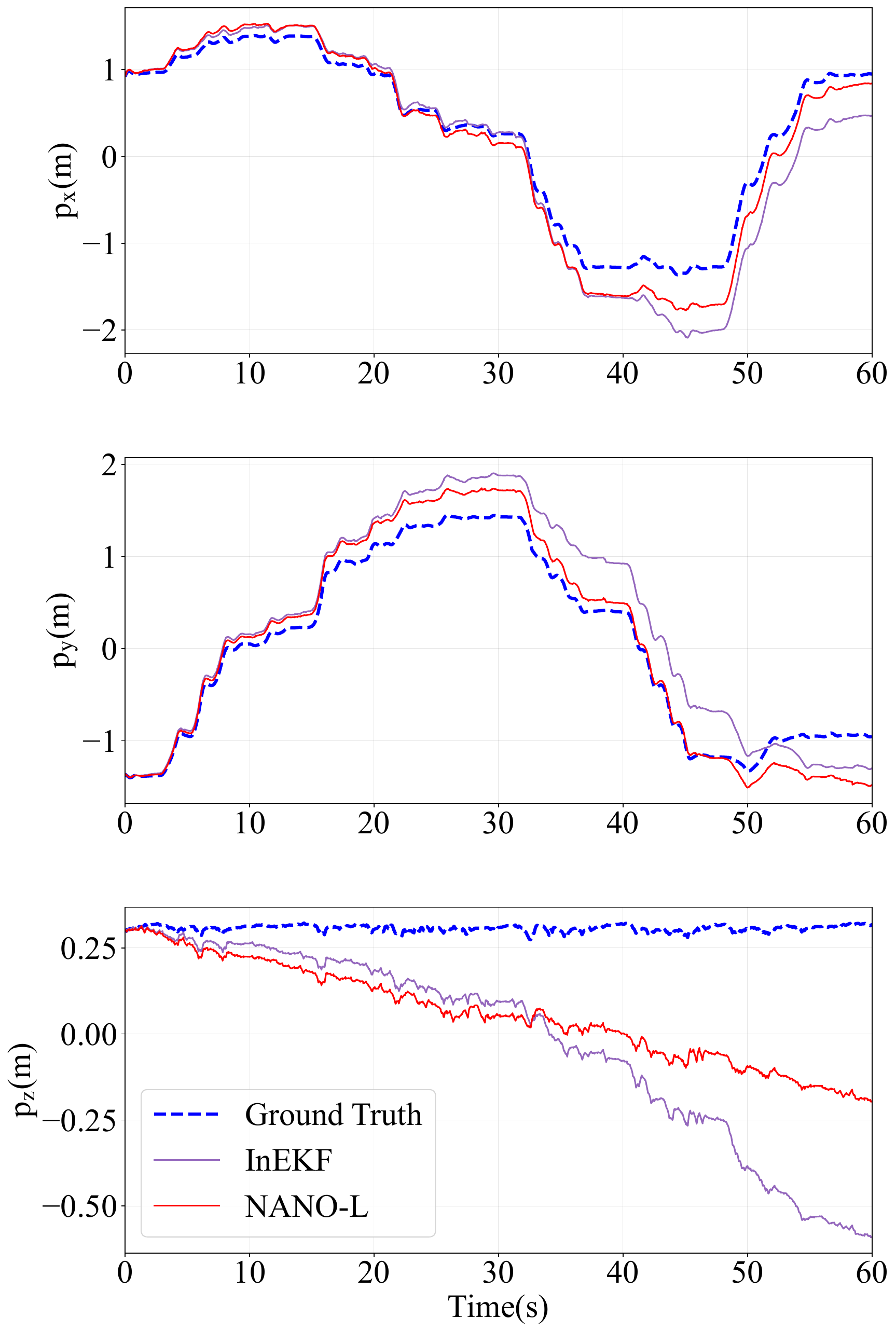}
    \label{fig:unstable_pos}
}
\subfloat{
    \includegraphics[width=0.33\textwidth]{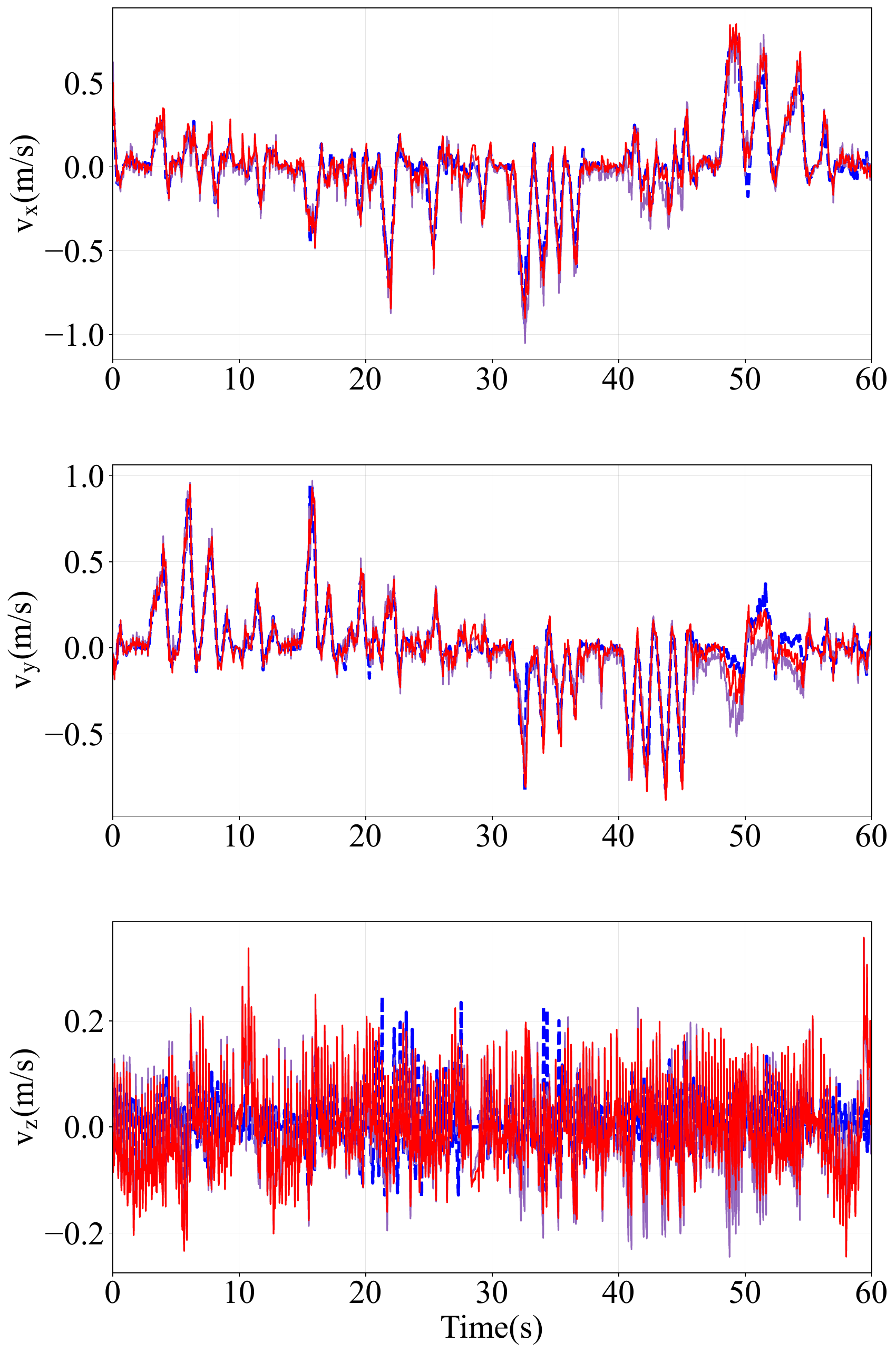}
    \label{fig:unstable_vel}
}
\subfloat{
    \includegraphics[width=0.33\textwidth]{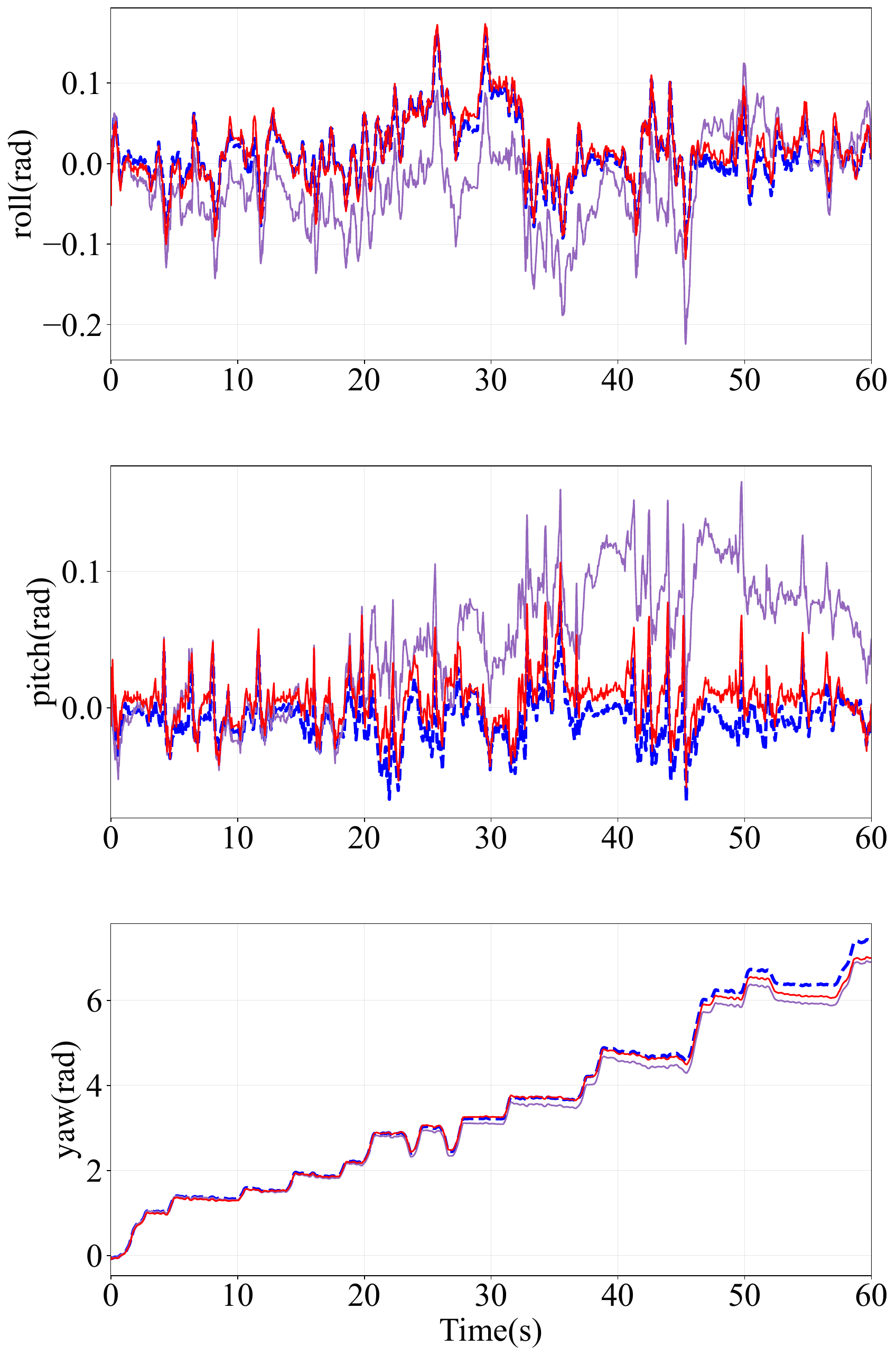}
    \label{fig:unstable_ori}
}
\caption{Estimated position, velocity, and orientation for InEKF and NANO-L on the unstable terrain.}
\label{fig:unstable}
\end{figure*}

\begin{figure*}[!t]
\centering
\subfloat{
    \includegraphics[width=0.33\textwidth]{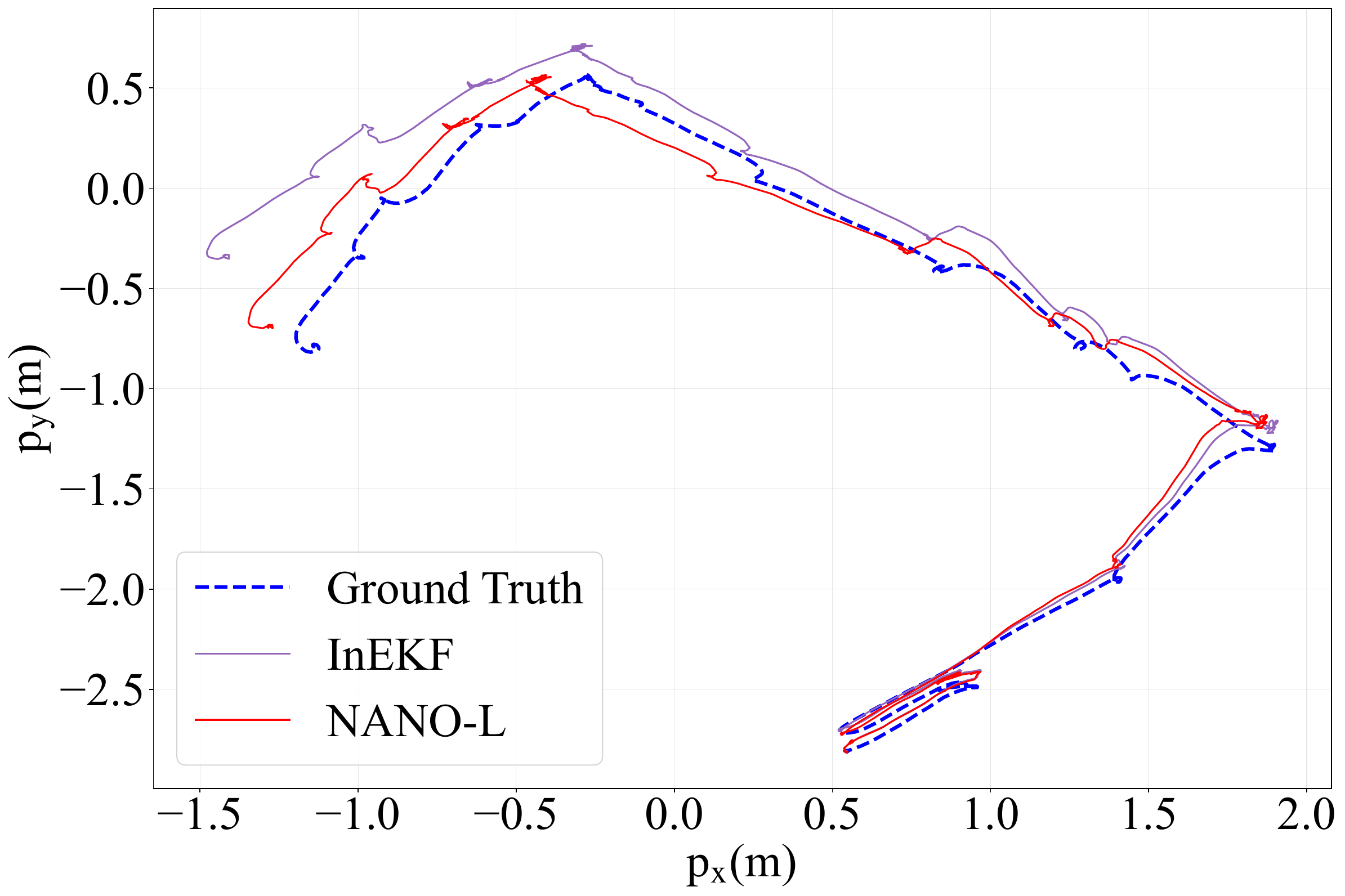}
}
\subfloat{
    \includegraphics[width=0.325\textwidth]{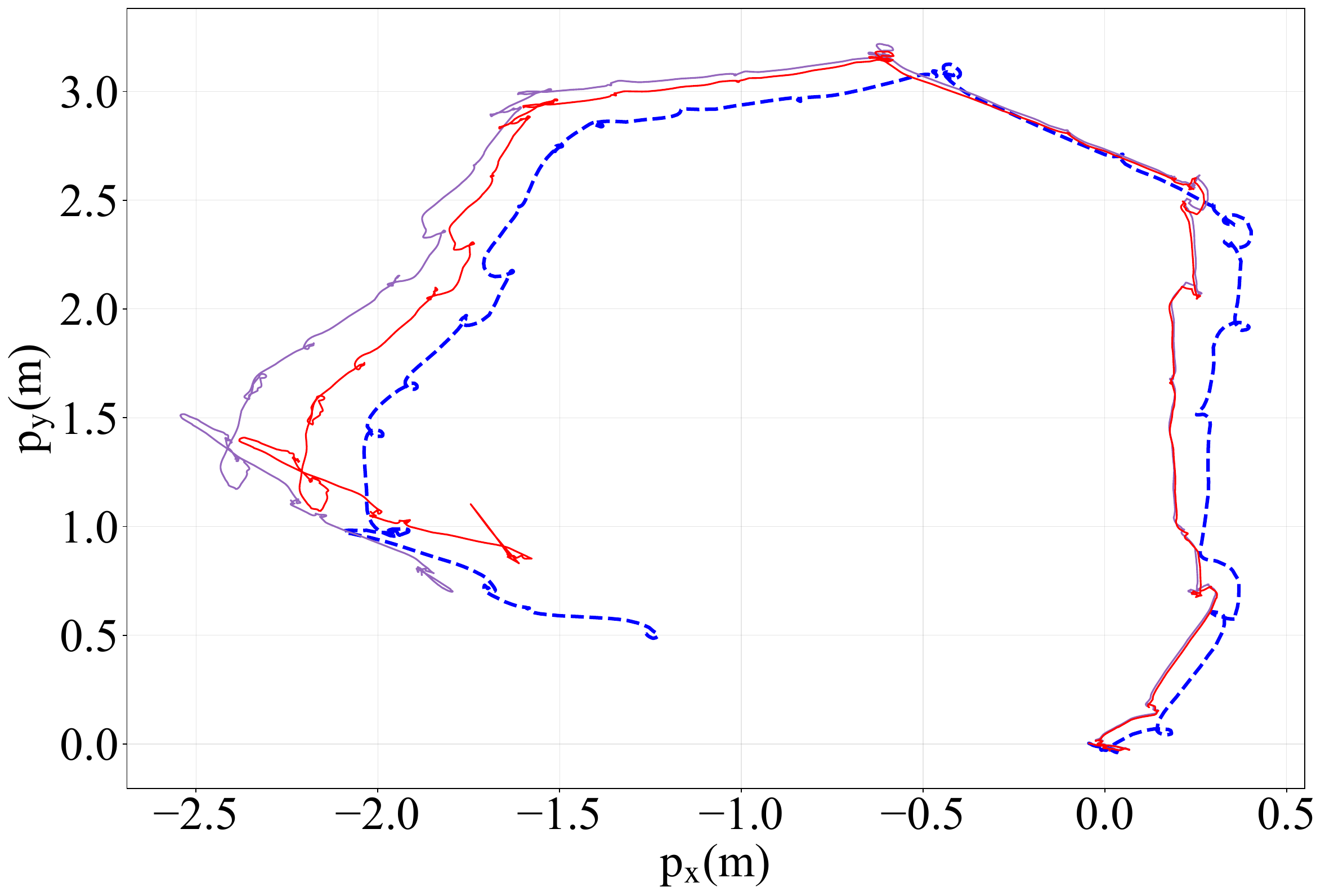}
}
\subfloat{
    \includegraphics[width=0.33\textwidth]{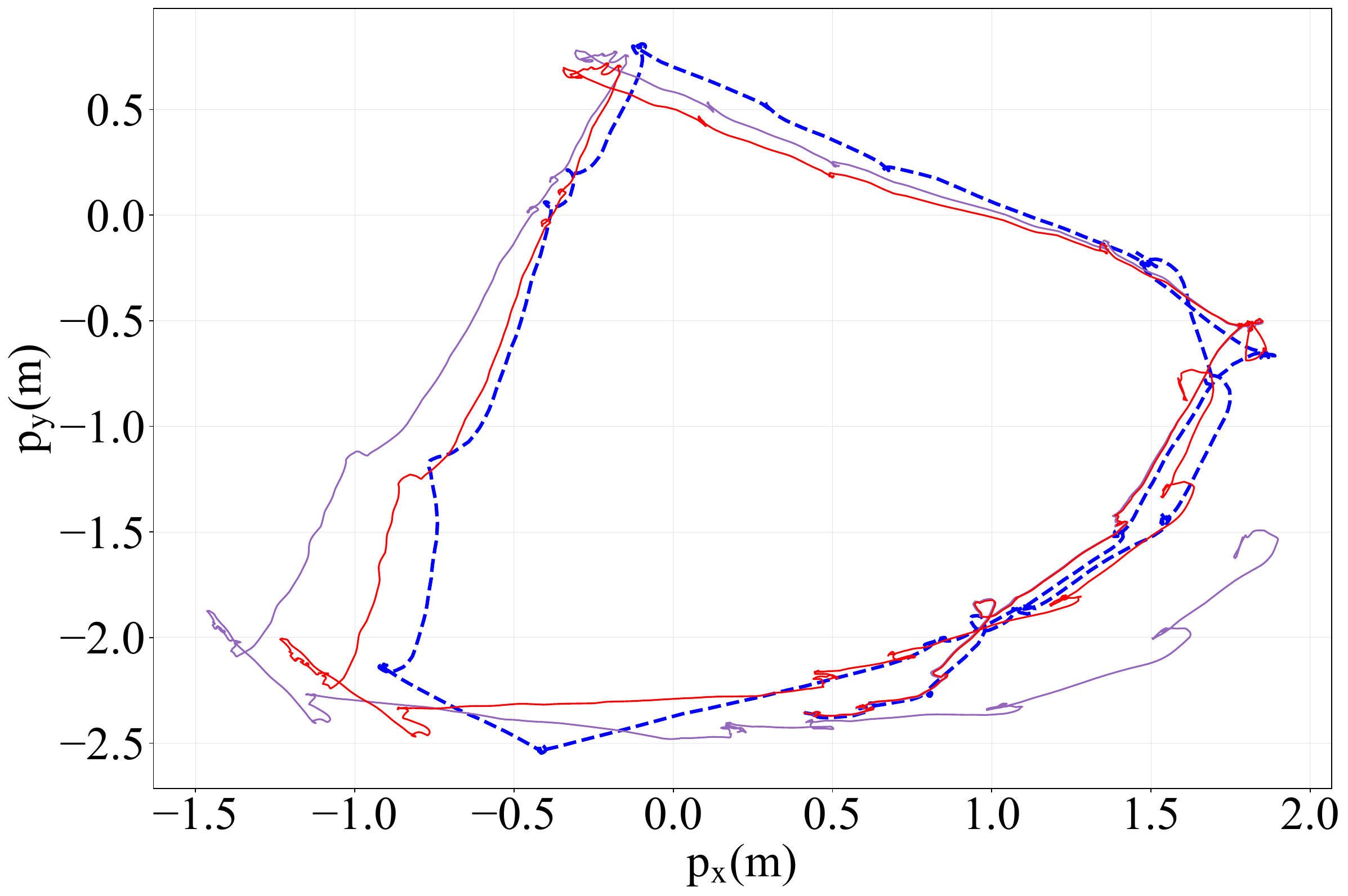}
}
\caption{Estimated trajectories from InEKF and NANO-L on the flat terrain with different distances and durations. From left to right: (9.25 m, 40 s), (10.25 m, 50 s), and (14.6 m, 60 s).}
\label{fig:traj}
\end{figure*}

\subsection{Humanoid Robot State Estimation}
We consider state estimation for a humanoid robot, where the estimator takes IMU measurements and leg odometry as inputs and outputs an estimate of the torso position.
We compare NANO with EKF and UKF baselines in terms of the position and velocity estimation errors.

\subsubsection{System Modeling}
We parameterize the base attitude by Euler angles
$\bm{\theta}_t := [\phi_t,\ \vartheta_t,\ \psi_t]^\top \in \mathbb{R}^3$
(roll--pitch--yaw), and denote the corresponding rotation matrix by
$\bm{R}(\bm{\theta}_t)\in\mathrm{SO}(3)$.
The discrete-time state is
\begin{equation}
\label{eq:g1_state_euler}
\bm{x}_t :=
\left(
\bm{p}_t,\ \bm{v}_t,\ \bm{\theta}_t,\ 
\bm{s}^L_t,\ \bm{s}^R_t,\ 
\bm{b}^a_t,\ \bm{b}^\omega_t
\right)\in\mathbb{R}^{21},
\end{equation}
where $\bm{p}_t,\bm{v}_t\in\mathbb{R}^3$ are the base position/velocity,
$\bm{s}^L_t,\bm{s}^R_t\in\mathbb{R}^3$ are the positions of the left/right foot contact points  in the world frame,
and $\bm{b}^a_t,\bm{b}^\omega_t\in\mathbb{R}^3$ are accelerometer and gyroscope biases.
The IMU input is $\bm{u}_t := (\tilde{\bm{a}}_t,\tilde{\bm{\omega}}_t)$. With sampling interval $\Delta t$, the IMU-driven dynamics can be written in the following discrete form
\begin{equation}
\label{eq:g1_process_block}
\begin{bmatrix}
\bm{p}_{t+1}\\
\bm{v}_{t+1}\\
\bm{\theta}_{t+1}\\
\bm{s}^L_{t+1}\\
\bm{s}^R_{t+1}\\
\bm{b}^a_{t+1}\\
\bm{b}^\omega_{t+1}
\end{bmatrix}
=
\begin{bmatrix}
\bm{p}_{t}+\bm{v}_{t}\Delta t\\
\bm{v}_{t}+\big(\bm{R}(\bm{\theta}_{t})\big(\tilde{\bm{a}}_{t}-\bm{b}^a_{t}\big)-\bm{g}\big)\Delta t\\
\bm{\theta}_{t}+\bm{T}(\bm{\theta}_{t})\big(\tilde{\bm{\omega}}_{t}-\bm{b}^\omega_{t}\big)\Delta t\\
\bm{s}^L_{t}\\
\bm{s}^R_{t}\\
\bm{b}^a_{t}\\
\bm{b}^\omega_{t}
\end{bmatrix}
+
\bm{\xi}_t.
\end{equation}
Here $\bm{\xi}_t\sim\mathcal{N}(\bm{0},\bm{Q}_t)$, and $\bm{T}(\bm{\theta})$ maps body angular velocity to Euler angle rates.
We model $\bm{s}^L,\bm{s}^R$ and the IMU biases as slow random walks and absorb their discrete-time perturbations into $\bm{\xi}_t$
through $\bm{Q}_t$.

Joint encoders provide joint angles $\bm{q}_t$ and velocities $\dot{\bm{q}}_t$.
Forward kinematics gives each foot position in the base frame $\tilde{\bm{r}}^i_t \in \mathbb{R}^3$.
Using the kinematic Jacobian, we also build a body-frame leg-velocity measurement $\tilde{\bm{\ell}}^i_t\in\mathbb{R}^3$.
In addition, we impose a foot-height constraint using a ground height $h_0$.

For each leg $i\in\{L,R\}$, the predicted measurement from the state is
\begin{equation}
\label{eq:g1_meas_leg}
\hat{\bm{y}}^i_t =
\begin{bmatrix}
\bm{R}_t^\top(\bm{s}^i_t-\bm{p}_t) \\
\bm{R}_t^\top\bm{v}_t \\
\bm{e}_3^\top \bm{s}^i_t
\end{bmatrix},
\qquad
\bm{y}^i_t =
\begin{bmatrix}
\tilde{\bm{r}}^i_t \\
\tilde{\bm{\ell}}^i_t \\
h_0
\end{bmatrix}
=
\hat{\bm{y}}^i_t + \bm{\zeta}^i_t,
\end{equation}
where $\bm{e}_3=[0,0,1]^\top$ and $\bm{\zeta}^i_t\sim\mathcal{N}(\bm{0},\bm{\Gamma}^i_t)$.
We stack the two legs to obtain a $14$-dimensional observation $\bm{y}_t=[(\bm{y}^L_t)^\top,(\bm{y}^R_t)^\top]^\top$.

To reduce the impact of contact switching and outliers, we use a simple contact indicator
\begin{equation}
\label{eq:g1_contact}
c^i_t = \mathbf{1}\Big( p^i_{z,t} < h_{\text{th}} \ \wedge\ \| \bm{v}^i_{xy,t}\| < v_{\text{th}}\Big),
\end{equation}
computed from the kinematically reconstructed foot height and slip speed.
When $c^i_t=0$, we inflate the measurement covariance of leg $i$ by a factor $(1+\alpha)$,
so that the filter downweights unreliable leg-odometry constraints during swing or slip.

\subsubsection{Experimental Setup}
We evaluate the estimators on a simulated humanoid trajectory containing both straight walking and turning motions.
For the sequence, IMU measurements and joint encoder data are recorded at 50\,Hz, while the simulator provides ground-truth  for evaluation.
All methods are initialized from the ground-truth state at $t=0$ and are run offline with fixed noise parameters throughout the sequence.
For NANO, we use a learning rate of $0.4$ and one inner iteration per step.
The contact-aware scaling in (\ref{eq:g1_contact}) uses thresholds $h_{\text{th}}=-0.03$, $v_{\text{th}}=0.3$, and inflation factor $\alpha = 10^5$.

\subsubsection{Comparison Results}
We report both global and local errors for position and velocity.

The quantitative results are reported in Table~\ref{tab:g1_pos_vel_metrics}.
Among the three methods, NANO achieves the best performance on all four metrics, where $\mathrm{ATE}_{pos}$ and $\mathrm{ATE}_{vel}$ represent ATE of position and velocity, and $\mathrm{RE}_{pos}$ and $\mathrm{RE}_{vel}$ represent the RE of position and velocity, respectively.
Fig.~\ref{fig:g1_vel} compares the estimated planar velocities.
While the main temporal pattern is captured by all methods, EKF and UKF exhibit stronger oscillations and larger local mismatches during rapid motion transitions.
In comparison, NANO tracks the reference velocity more accurately.

Fig.~\ref{fig:g1_pos} presents the estimated $x$--$y$ trajectories.
All three methods recover the overall motion trend, but clear differences can be observed in estimation accuracy.
NANO remains closest to the ground-truth path over the entire sequence, whereas EKF and UKF show larger deviations, particularly in the turning region and the upper-left part of the trajectory.

\begin{table}[!t]
\fontsize{9}{9}\selectfont
\caption{Comparison of ATE and RE on the simulated trajectory.}
\label{tab:g1_pos_vel_metrics}
\centering
\renewcommand{\arraystretch}{1.15}
\setlength{\tabcolsep}{3pt}
\begin{tabular}{c|cccc}
\toprule[1pt]
Method
& $\mathrm{ATE}_{pos}$ [m]
& $\mathrm{ATE}_{vel}$ [m/s]
& $\mathrm{RE}_{pos}$ [m]
& $\mathrm{RE}_{vel}$ [m/s] \\
\midrule[0.5pt]
EKF
& 0.386 & 0.346 & 0.240 & 0.549 \\

UKF
& 0.387 & 0.345 & 0.239 & 0.551 \\

NANO
& \textbf{0.101} & \textbf{0.267} & \textbf{0.187} & \textbf{0.370} \\
\bottomrule[1pt]
\end{tabular}
\end{table}

\begin{figure}[!t]
\centering
\includegraphics[width=0.48\textwidth]{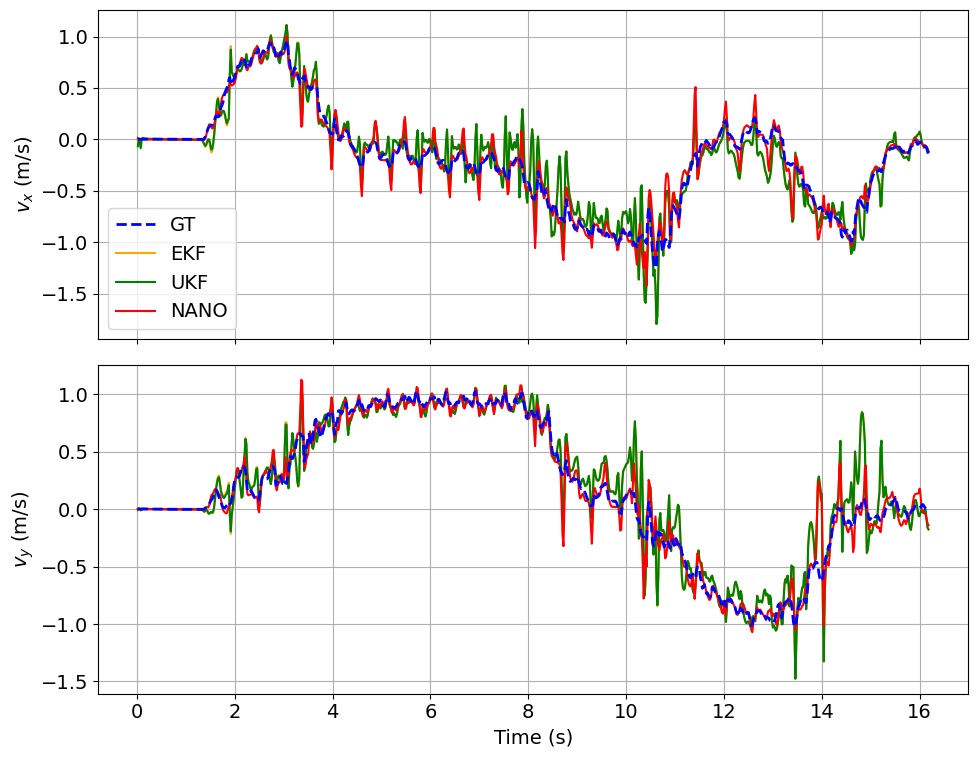}
\caption{Estimated velocity from EKF, UKF and NANO.}
\label{fig:g1_vel}
\end{figure}

\begin{figure}[!t]
\centering
\includegraphics[width=0.48\textwidth]{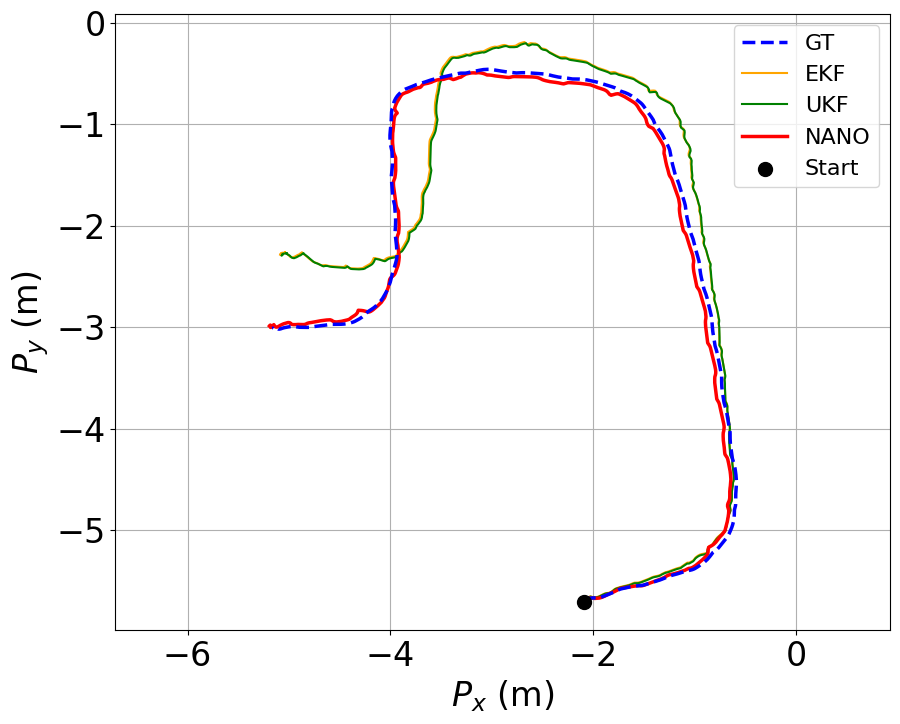}
\caption{Estimated trajectories from EKF, UKF and NANO.}
\label{fig:g1_pos}
\end{figure}

\section{Conclusion}\label{sec.conclusion}
This tutorial revisits Gaussian Bayesian filtering from an information-geometric perspective and presents a unified optimization interpretation of the prediction and update steps. By viewing filtering as variational inference over probability distributions, Gaussian filters can be understood as projecting the true Bayesian posterior onto the manifold of Gaussian distributions. 
Within this framework, we introduced the NANO filter, which performs the measurement update through natural gradient descent on the statistical manifold, yielding geometry-aware updates that respect the intrinsic structure of Gaussian distributions and remain invariant to parameterization. The resulting algorithm directly targets the stationary conditions of the variational formulation and avoids the structural limitations of traditional linearization-based Gaussian filters. 
We further demonstrated that, in the linear–Gaussian case, a single natural-gradient step exactly recovers the classical Kalman filter update, thereby establishing a clear theoretical link to existing methods. 
Practical considerations, including numerical expectation evaluation, covariance stabilization, and extensions to manifold-valued states, were also discussed to facilitate implementation. Overall, the information-geometric viewpoint provides an intuitive and principled framework for understanding and designing Gaussian filters for nonlinear dynamical systems.

\bibliographystyle{ieeetr}
\bibliography{
ref
}


\end{document}